\documentclass[runningheads]{llncs}

\usepackage[mobile]{eccv}

\usepackage{eccvabbrv}

\usepackage{booktabs}
\usepackage{tabularx} 
\usepackage{lipsum} 

\usepackage[accsupp]{axessibility}  
\usepackage{graphicx}
\usepackage{tikz}
\usepackage{comment}
\usepackage{amsmath,amssymb} 
\usepackage{color}
\usepackage{listings}
\usepackage{booktabs}
\usepackage{multirow}
\usepackage{colortbl}
\usepackage{xcolor}
\usepackage{graphicx}
 \usepackage{adjustbox}
\usepackage{wrapfig}
\usepackage{caption}
\usepackage{comment}
\usepackage{cancel}
\usepackage{rotating} %
\usepackage{pifont}
\DeclareMathOperator*{\argmax}{arg\,max}
\DeclareMathOperator*{\argmin}{arg\,min}

\DeclareMathOperator*{\E}{\mathbb{E}}
\usepackage{multirow}
\usepackage[accsupp]{axessibility}  
\usepackage{appendix}

\usepackage[colorlinks,linkcolor=blue,urlcolor=blue, citecolor=eccvblue]{hyperref}
\usepackage{orcidlink}

\begin{document}
\title{Augmented Neural Fine-Tuning for Efficient Backdoor Purification} 

\titlerunning{Augmented Neural Fine-Tuning for Efficient Backdoor Purification}

\author{Nazmul Karim\inst{*,1}\orcidlink{0000-0001-5522-4456} \and
Abdullah Al Arafat\inst{*,2}\orcidlink{0000-0002-7017-0158} \and
Umar Khalid\inst{1}\orcidlink{0000-0002-3357-9720} \and
Zhishan Guo\inst{2}\orcidlink{0000-0002-5967-1058} \and
Nazanin Rahnavard\inst{1}\orcidlink{0000-0003-3434-1359}}

\authorrunning{N. Karim et al.}

\institute{$^1$University of Central Florida \quad 
$^2$North Carolina State University}

\maketitle
\begingroup
\renewcommand\thefootnote{$*$}
\footnotetext{Equal Contribution}
\endgroup

\begin{abstract}
Recent studies have revealed the vulnerability of deep neural networks (DNNs) to various backdoor attacks, where the behavior of DNNs can be compromised by utilizing certain types of triggers or poisoning mechanisms. State-of-the-art (SOTA) defenses employ too-sophisticated mechanisms that require either a computationally expensive adversarial search module for reverse-engineering the trigger distribution or an over-sensitive hyper-parameter selection module. Moreover, they offer sub-par performance in challenging scenarios, \eg, limited validation data and strong attacks. In this paper, we propose---{\em Neural mask Fine-Tuning (NFT)}---with an aim to optimally re-organize the neuron activities in a way that the effect of the backdoor is removed. Utilizing a simple data augmentation like MixUp, NFT relaxes the trigger synthesis process and eliminates the requirement of the adversarial search module. Our study further reveals that direct weight fine-tuning under limited validation data results in poor post-purification clean test accuracy, primarily due to \emph{overfitting issue}. To overcome this, we propose to fine-tune neural masks instead of model weights. In addition, a \emph{mask regularizer} has been devised to further mitigate the model drift during the purification process. The distinct characteristics of NFT render it highly efficient in both runtime and sample usage, as it can remove the backdoor even when a single sample is available from each class. We validate the effectiveness of NFT through extensive experiments covering the tasks of image classification, object detection, video action recognition, 3D point cloud, and natural language processing. We evaluate our method against 14 different attacks (LIRA, WaNet, \etc) on 11 benchmark data sets (ImageNet, UCF101, Pascal VOC, ModelNet, OpenSubtitles2012, \etc). Our code is available online in this \href{https://github.com/nazmul-karim170/NFT-Augmented-Backdoor-Purification}{GitHub Repository}.
\end{abstract}
\section{Introduction}
Machine learning and computer vision algorithms are increasingly common in safety-critical applications~\cite{janai2020computer,bian2022machine}, necessitating the design of secure and robust learning algorithms.  
Backdoor attack~\cite{chen2017targeted,gu2019badnets} on deep neural network (DNN) models is one of the heavily studied branches of AI safety and robustness. Backdoor defenses can generally be categorized into two major groups based on whether the defense is done during training or test. Training-time-defense (\eg,~\cite{li2021anti,tran2018spectral,levine2020deep}) focuses on training a benign model on the poisonous data, while test-time-defense (\eg,~\cite{liu2018fine,wu2021adversarial,chai2022one,li2023reconstructive}) deals with purifying backdoor model after it has already been trained. In this work, our goal is to develop an efficient test-time defense. Some test-time defenses focus on synthesizing trigger patterns~\cite{wang2019neural,chen2019deepinspect,wang2020practical} followed by vanilla weight fine-tuning. Most of these defenses aim to synthesize class-specific triggers independently or use additional models to generate the triggers simultaneously. Recent state-of-the-art (SOTA) backdoor defense methods, \eg, ANP~\cite{wu2021adversarial}, I-BAU~\cite{zeng2021adversarial}, AWM~\cite{chai2022one}, employ very similar techniques and do not work well without an expensive \emph{adversarial search module}. For example, ANP~\cite{wu2021adversarial} performs an adversarial search to find vulnerable neurons responsible for backdoor behavior. Identifying and pruning vulnerable neurons require an exhaustive adversarial search, resulting in high computational costs.
Similar to ANP, AWM and I-BAU also resort to trigger synthesizing with a modified adversarial search process. Another very recent technique, FT-SAM~\cite{zhu2023enhancing} uses sharpness-aware minimization (SAM)~\cite{foret2021sharpnessaware} to fine-tune model weights. SAM is a recently proposed optimizer that utilizes Stochastic Gradient Descent (SGD), which penalizes sudden changes in the loss surface by constraining the search area to a compact region. Since SAM performs a double forward pass to compute the loss gradient twice, it results in a notable runtime increase for FT-SAM. In our work, \emph{we aim to develop an effective backdoor defense system that neither requires an expensive adversarial search process to recover the trigger nor a special type of optimizer with a runtime bottleneck.}


To achieve this goal, we propose a simple yet effective approach \underline{N}eural mask \underline{F}ine-\underline{T}uning (NFT), to remove backdoor through augmented fine-tuning of cost-efficient neural masks. We start with replacing the expensive adversarial search-based \emph{trigger synthesis} process with a simple data augmentation technique---MixUp~\cite{zhang2017mixup}. In general, the backdoor is inserted by forcing the model to memorize the trigger distribution. Intuitively,  synthesizing and unlearning that trigger distribution would effectively remove the backdoor. In this work, we show that \emph{unlearning can be performed by simply optimizing the MixUp loss over a clean validation set}. Our theoretical analysis suggests that MixUp loss is an upper bound on the standard loss obtained from triggered (synthesized or already known) validation data, termed as \emph{ideal purification loss}. As the minimization of ideal loss guarantees backdoor purification, minimizing the MixUp loss would effectively remove the backdoor (Sec.~\ref{sec:optimization_NFT}). {\color{black} As the next step of our method, we address the overfitting issue during weight fine-tuning under \emph{limited validation data}. In general, the outcome of such overfitting is poor post-purification test accuracy, which is not desirable for any backdoor defense.} To this end, we propose to fine-tune a set of neural masks instead of the model weights, as this type of soft-masking enables us to reprogram the neurons affected by the backdoor without significantly altering the original backdoor model. As an added step to this, a mask regularizer has been introduced to further mitigate model drift during the purification process. In addition, we deploy a mask scheduling function to have better control over the purification process. Our experimental results indicate that these straightforward yet intuitive steps significantly improve the post-purification test accuracy as compared to previous SOTA. Our contributions can be summarized as follows:
\begin{itemize}
    \item We propose a novel backdoor removal framework utilizing simple MixUp-based model fine-tuning. Our thorough analysis shows how minimizing the MixUp loss eliminates the requirement of an expensive trigger synthesis process while effectively removing the backdoor (Sec.~\ref{sec:optimization_NFT}).  
    \item To preserve the post-purification test accuracy, we propose to fine-tune soft neural masks (instead of weights) as it prevents any drastic change in the original backdoor model (Sec.~\ref{sec:clean_NFT}). Additionally, a novel mask regularizer has been introduced that further encourages the purified model to retain the class separability of the original model. In addition to being computationally efficient, our proposed method shows significant improvement in sample efficiency as it can purify backdoor even with one-shot fine-tuning, i.e., only a single sample is available from each class (Sec.~\ref{sec:sample_eff_NFT}). 
    \item To show the effectiveness of NFT, we perform an extensive evaluation with 11 different datasets. Compared to previous SOTA, the superior performance against a wide range of attacks suggests that augmentation like MixUp can indeed replace the \emph{trigger synthesis} process (Sec.~\ref{sec:exp_result}). 
\end{itemize}

\section{Related Work}
\noindent{\bf Backdoor Attack.} Neural networks are intrinsically vulnerable to backdoor attacks~\cite{manoj2021excess,xian2023understanding}. A substantial number of studies have investigated the possibility of backdoor attacks after the initial studies~\cite{gu2019badnets,chen2017targeted,liu2017neural} found the existence of backdoors in DNNs. Generally, backdoor attacks are categorized into two types: clean-label attacks and poison-label attacks. A clean-label backdoor attack does not alter the label~\cite{ning2021invisible,turner2019cleanlabel,zhao2020clean}, while a poison-label attack aims at specific target classes such that the DNN misclassifies to those classes in the presence of a trigger~\cite{li2021neural}. As for trigger types, researchers have studied numerous types of triggering patterns in their respective attacks~\cite{gu2019badnets,chen2017targeted,li2020backdoor,edraki2021odyssey}. Such triggers can exist in the form of dynamic patterns~\cite{li2020backdoor} or as simple as a single pixel~\cite{tran2018spectral}. Some of the more complex backdoor triggers that have been proposed in the literature are sinusoidal strips~\cite{barni2019new}, 
 adversarial patterns~\cite{zhao2020clean}, and blending backgrounds~\cite{chen2017targeted}. Besides, backdoor attacks exist for many different tasks, \eg, multi-label clean image attack~\cite{chen2023clean}  has been proposed that alters the label distribution to insert triggers into the model, which works well in multi-label (\eg, detection) settings; domain adaptation~\cite{ahmed2023ssda} setting while adversary source can successfully insert backdoor to the target domains, \etc.

\noindent\textbf{Backdoor Defense.} Generally, the backdoor defense methods are categorized into two types: Training Time Defense, and Test Time Defense. Regarding training time defense techniques (a few to mention~\cite{tran2018spectral,gao2019strip,liu2017neural,hong2020effectiveness}), the researchers have proposed numerous defense methods through input pre-processing \cite{liu2017neural}, poison-suppression~\cite{hong2020effectiveness}, model diagnosis~\cite{li2020backdoor}, network pruning~\cite{wu2020adversarial,liu2017neural}, and model reconstruction~\cite{zhao2020bridging}, \etc. Notably, DeepSweep~\cite{qiu2021deepsweep} explores different augmentations to purify a backdoor model and rectify the triggered samples. Although DeepSweep revealed that different augmentation functions could be leveraged to invert the backdoor effect of a model or erase the trigger from a trigger-embedded image, our work re-purposes the usage of augmentation differently to cover the approximate (unknown) trigger distribution during the purification phase. Moreover,  DeepSweep assumes that backdoor triggers are known to the defender, which is hardly a practical assumption. In the case of test time defenses, besides the works mentioned in the introduction related to reverse-engineering of backdoor triggers in the input samples~\cite{wang2019neural,chen2019deepinspect,wang2020practical}, several recent works explored the model vulnerability/sensitivity towards adversarially perturbed neurons~\cite{wu2021adversarial}, weights~\cite{chai2022one}, or network channels~\cite{zheng2022data}. However, these approaches require expensive adversarial search processes to be effective. A concurrent work FIP~\cite{FIP} studied the loss-surface smoothness of the backdoor model and developed a purification method by regularizing the spectral norm of the model.  
\section{Threat Model}
\noindent\textbf{Attack Model.} Our work considers the most commonly used data poisoning attacks. Consider $\mathbb{D}_{\mathrm{train}} = \{x_i, y_i\}_{i=1}^{N}$ as the training data where $x_i \in \mathbb{R}^{d}$ is an input sample labeled as $y_i\in \{0,\ldots,c-1\}$ sampled from unknown distribution $\mathcal{D}$ of the task to be learned. Here, $N$ is the total number of samples, and $d$ is the dimension of the input sample. In addition, we assume there are $c$ number of classes in the input data. Let $f_{\theta^*}: \mathbb{R}^d \rightarrow \mathbb{R}^c$ be a benign (ideal) DNN trained with $\mathbb{D}_{\mathrm{train}} \sim \mathcal{D}$. 
Here, $\theta^*$ is the DNN parameters that is to be optimized using a suitable loss function $\ell(.,.)$. The total empirical loss can be defined as, 
\begin{equation}\label{eq:ce_loss}
    \mathcal{L}({\theta^*}, \mathbb{D}_{\mathrm{train}}) =\frac{1}{N} \sum_{i=1}^N [\ell(y_i, f_{\theta^*}(x_i))].    
\end{equation} 
Now, consider an adversary inserts backdoor to a model $f_\theta(.)$ through modifying a small subset of $\mathbb{D}_{\mathrm{train}}$ as $\{\hat{x}_i, \hat{y}_i\}$ such that $\hat{y}_i = \argmax f_\theta ( \hat{x}_i )$ preserving $y_i =\argmax f_\theta (x_i),~ \forall (x_i,y_i) \in \mathbb{D}_{\mathrm{train}}$. Here, $\hat{x}_i = x_i +\delta$ is the triggered input with adversary set target label $\hat{y}_i \neq y_i$, where $\delta \in \mathbb{R}^d$ represents trigger pattern.

\noindent\textbf{Defense Objective.} Consider a defense model where defender removes backdoor from $f_\theta(.)$ using a small validation data $\mathbb{D}_\mathrm{val} = \{x_i, y_i\}_{i=1}^{N_{val}}$ such that $y_i= \argmax f_{\theta_c} (\hat{x}_i)$, where $y_i \neq \hat{y}_i$.

\section{Neural Fine-Tuning (NFT)}\label{sec:method}

Let us consider a fully-connected DNN, $f_{\theta}:\mathbb{R}^d \rightarrow \mathbb{R}^c$, that receives a datapoint $x \in \mathbb{R}^d$ and predicts a probability distribution $p \in \mathbb{R}^c$; where $c$ is the number of classes. In general, $x$ goes through a multi-layer DNN architecture before the DNN model predicts an output class $i = \argmax p$. Let us consider a multi-layer DNN architecture with $L$ layers in which the $l$-th layer contains $k_l$ neurons. The neurons of each layer produce activations $\psi_l \in \mathbb{R}^{k_l}$ based on the output activations of previous layer $\psi_{l-1} \in \mathbb{R}^{k_{l-1}}$. To be specific, 
\begin{equation}\label{eq:dnn_operations}
    \psi_l := \sigma(\Theta_l^T\cdot\psi_{l-1}+b_l),
\end{equation} 
where $\sigma(.)$ is a non-linear activation function, matrix $\Theta_l = [\Theta_l^{(1)} \cdots \Theta_l^{(k_l)}] \in \mathbb{R}^{k_{l-1} \times k_l}$, 
for $l=1,2,\ldots, L$, includes the weights of the $l$-th layer, and $b_l \in \mathbb{R}^{k_l}$ is the bias vector. Here, $\Theta_l^{(j)} \in \mathbb{R}^{k_{l-1}}$ denotes the weights vector corresponding to the activation of the $j$-th neuron of the $l$-th layer. Model parameters $\theta$ can be expressed as $\theta = \{\Theta_1, \ldots, \Theta_L\}$. This type of multi-layer DNN architecture is also valid for convolutional neural networks, where we use multiple 2-D arrays of neurons (\ie, filters) instead of a 1-D array.

\subsection{Backdoor Suppressor} \label{sec:optimization_NFT}

Our objective is to make the backdoor model forget about poison distribution while retaining the knowledge of clean distribution. To understand how we achieve this objective, let us first revisit the process of generating triggered data ($\hat{x}$). In general, $\hat{x}$ is created by adding minor modifications (\ie, adding triggers $\delta$) to clean data ($x$). Note that the backdoor is inserted by forcing the model to learn the mapping, $\hat{x} \rightarrow \hat{y}$. Here, ($\hat{x},\hat{y}$) is the poison data. If we were to change the mapping from ($\hat{x} \rightarrow \hat{y}$) to ($\hat{x} \rightarrow y$), we would have a robust clean model instead of a backdoor model. This is because, in this case, the model would treat $\delta$ as one type of augmentation and $\hat{x}$ as the augmented clean data. In summary, the backdoor insertion process functions as an augmentation process if we simply use $y$ instead of $\hat{y}$.  
Now, ideally, fine-tuning the backdoor model with triggered data with corresponding ground truth labels (\ie, \(\{\hat{x}, y\}\)) would remove the backdoor effect from the model. Let us consider this ideal scenario where we have access to the trigger $\delta$ during purification. We define the \emph{ideal purification loss} as:
\begin{equation}\label{eqn:ideal-loss}
   \mathcal{L}^{\mathrm{ideal}} (\theta, \mathbb{D}_\mathrm{val}) = \frac{1}{N_{\mathrm{val}}} \sum_{i=1}^{N_\mathrm{val}}\ell(y_i, f_\theta(\hat{x}_i)),
\end{equation} 
where \(\hat{x} = x + \delta\) and $y$ is the ground truth label of $x$. In our work, as we do not have access to $\delta$ or adversarially reverse engineer it~\cite{wu2021adversarial,zeng2021adversarial,chai2022one}, we relax this process by strongly augmenting the clean validation data, \ie, creating augmented $\mathbb{D}_\mathrm{val}$ with already known augmentation technique such as MixUp~\cite{zhang2017mixup}. For MixUp, we can easily perform \(\tilde{x}_{i,j} = \lambda x_i + (1-\lambda)x_j\) and \(\tilde{y}_{i,j} = \lambda y_i + (1-\lambda) y_j\) for \(\lambda \in [0,1]\); here $\tilde{y}_{i,j}$ represents the linear combination of one-hot vectors corresponding to $y_i$ and $y_j$. The loss after the MixUp becomes:
\begin{equation}\label{eq:aug_nft-NO}
    ~\mathcal{L}^{\mathrm{mix}}(\theta,\mathbb{D}_\mathrm{val} )=\frac{1}{N_\mathrm{val}^2} \sum_{i,j=1}^{N_\mathrm{val}}\E_{\lambda \sim \mathcal{D_\lambda}}\ell(\tilde{y}_{i,j}, f_{\theta}(\Tilde{x}_{i,j})),
\end{equation} 
where \(\mathcal{D}_\lambda\) is a distribution supported on \([0,1]\). In our work, we consider the widely used $\mathcal{D}_\lambda$ -- Beta distribution $Beta(\alpha,\beta)$ for $\alpha,\beta>0$. We provide both empirical (Sec.~\ref{sec:exp_result}) and theoretical proof for a binary classification problem on why minimizing Eq.~\eqref{eq:aug_nft-NO} would effectively remove the backdoor.

\noindent \textbf{Theoretical Justifications.}\label{sec:theor_just} For a fully-connected neural network (NN) with logistic loss \(\ell (y, f_\theta(x)) = \log(1+\exp{(f_\theta(x))}) - y f_\theta(x)\) with \(y \in \{0,1\}\),  it can be shown that $\mathcal{L}^{\mathrm{mix}}(\theta,\mathbb{D}_\mathrm{val} )$ is an upper-bound of the second order Taylor expansion of the ideal loss $\mathcal{L}^{\mathrm{ideal}} (\theta, \mathbb{D}_\mathrm{val})$. With the nonlinearity $\sigma$ for ReLU and max-pooling in NN, the function $f_\theta$ satisfies that $f_\theta(x) = \nabla f_\theta(x)^T x$ and $\nabla^2 f_\theta(x) = 0$ almost everywhere, where the gradient is taken with respect to the input $x$. 

We first rewrite the $\mathcal{L}^{\mathrm{ideal}} (\theta, \mathbb{D}_\mathrm{val})$ using Taylor series approximation. The second-order Taylor expansion of $\ell(y, f_\theta(x + \delta))$ is given by,
\begin{equation*}
    \ell(y, f_\theta(x + \delta)) = \ell(y, f_\theta(x)) + (g(f_\theta(x)) -y)(f_\theta(\delta)) + \frac{1}{2}g(f_\theta(x))(1-g(f_\theta(x)))(f_\theta(\delta))^2,
\end{equation*} 
where \(g(x) = \frac{e^x}{1+e^x}\) is the logistic function. Based on the MixUp related analysis in prior works~\cite{carratino2022mixup,zhang2020does}, the following can be derived for \(\mathcal{L}^{\mathrm{mix}}(\theta,\mathbb{D}_\mathrm{val} )\) using the second-order Taylor series expansion,
\begin{lemma}\label{lemma:mixup_reg}
Assuming $f_\theta(x) = \nabla f_\theta(x)^T x$ and $\nabla^2 f_\theta(x) = 0$ (which are satisfied by ReLU and max-pooling activation functions), \(\mathcal{L}^{\mathrm{mix}}(\theta,\mathbb{D}_\mathrm{val} )\) can be expressed as,
\begin{equation}
    \mathcal{L}^{\mathrm{mix}}(\theta,\mathbb{D}_\mathrm{val} ) = \mathcal{L}(\theta,\mathbb{D}_\mathrm{val} ) + \mathcal{R}_1(\theta, \mathbb{D}_\mathrm{val}) +\mathcal{R}_2(\theta, \mathbb{D}_\mathrm{val})
\end{equation}  
where,
\[
\mathcal{R}_1(\theta, \mathbb{D}_\mathrm{val}) \geq \frac{Rc_x\E_\lambda[(1-\lambda)]\sqrt{d}}{N_\mathrm{val}} \sum_{i=1}^{N_\mathrm{val}}|g(f_\theta(x_i)) - y_i |\cdot||\nabla f_\theta(x_i)||_2
\]
\[
\mathcal{R}_2(\theta, \mathbb{D}_\mathrm{val}) \geq \frac{R^2c_x^2\E_\lambda[(1-\lambda)]^2{d}}{2N_\mathrm{val}} \sum_{i=1}^{N_\mathrm{val}}|g(f_\theta(x_i))(1 - g(f_\theta(x_i)))|\cdot||\nabla f_\theta(x_i)||_2^2,
\]
where \(R=\min_{i\in [N_\mathrm{val}]}\langle\nabla f_\theta(x_i), x_i\rangle/||\nabla f_\theta(x_i)||\cdot||x_i||\) and \(c_x > 0\) is a constant. 
\end{lemma}

By comparing $\ell(y, f_\theta(x + \delta))$ and $\mathcal{L}^{\mathrm{mix}}(\theta,\mathbb{D}_\mathrm{val} )$ for a fully connected NN, we can prove the following.

\begin{theorem}
Suppose that $f_\theta(x) = \nabla f_\theta(x)^T x$, $\nabla^2 f_\theta(x) = 0$ and there exists a constant $c_x > 0$ such that $\|x_i\|_2 \geq c_x \sqrt{d}$ for all $i \in \{1, \ldots, N_{\mathrm{val}}\}$. Then, for any $f_\theta$, we have
\[
\mathcal{L}^{\mathrm{mix}}(\theta,\mathbb{D}_\mathrm{val}) \geq \frac{1}{N_{\mathrm{val}}} \sum_{i=1}^{N_{\mathrm{val}}} \ell\left(y_i, f_\theta(x_i+\varepsilon_i)\right) \geq \frac{1}{N_{\mathrm{val}}} \sum_{i=1}^{N_{\mathrm{val}}} \ell\left(y_i, f_\theta(x_i+\varepsilon)\right)
\]
where $\varepsilon_i = R_i c_x \E_{\lambda \sim \mathcal{D}_{\lambda}}[1 - \lambda]\sqrt{d}$ with $R_i = \langle\nabla f_\theta(x_i), x_i\rangle/||\nabla f_\theta(x_i)||\cdot||x_i||$ and $\varepsilon = \min\{\varepsilon_i\}$.    
\end{theorem}

\begin{proof}
   is provided in \emph{Supplementary}. 
\end{proof}

Theorem 1 implies that as long as \(||\delta|| \leq \varepsilon\) holds, the MixUp loss \(\mathcal{L}^{\mathrm{mix}}(\theta,\mathbb{D}_\mathrm{val})\) can be considered as an upper-bound of \(\mathcal{L}^{\mathrm{ideal}}(\theta,\mathbb{D}_\mathrm{val})\). 

\subsection{Clean Accuracy Retainer}\label{sec:clean_NFT}
In practice, it is desirable for a backdoor defense technique to be highly \emph{runtime} efficient and retain the \emph{clean test accuracy} of the original model. For better runtime efficiency and to retain clean accuracy, we choose to apply neural mask fine-tuning instead of fine-tuning the entire model, which can be formulated as, 
\begin{equation}\label{eq:objective_function}
    \widehat{M} = \argmin_{M~\mid~m_l^{(i)}\in [\mu (l),1],~\forall l,i} \mathcal{L}^{\mathrm{mix}}({\theta \odot M},\mathbb{D}_{\mathrm{val}}). 
\end{equation} 
We only optimize for neural masks  $M$ using $\mathbb{D}_\mathrm{val}$ and define $\theta \odot M$ as, 
\begin{equation}
\theta \odot M := \{\Theta_1\odot M_1, \Theta_2 \odot M_2, \ldots, \Theta_L \odot M_L\},
\end{equation} 
where $\theta := \{\Theta_1, \ldots, \Theta_L\}$, $M := \{M_1, \ldots, M_L\}$, and $M_l = [m_l^{(1)} \cdots m_l^{(k_l)}]^T\in \mathbb{R}^{k_l}$. 
We have
\begin{equation}~\label{eq:fin_mask}
\Theta_l\odot M_l := [m_l^{(1)}\Theta_l^{(1)} \cdots m_l^{(k_l)}\Theta_l^{(k_l)}]\in \mathbb{R}^{k_{l-1}\times k_{l}}.
\end{equation} 
Note, in Eq.~\eqref{eq:fin_mask}, a \textbf{scalar mask} $m_l^{(i)}$ is applied to the weight vector $\Theta_l^{(i)}$ corresponding to the $i^{th}$ neuron of $l^{th}$ layer. In our work, we formulate a constraint optimization problem where $M$ depends on a mask scheduling function $\mu(l): [1, L] \to [0,1] $. 
Notice that  $\mu(l)$ provides the lower limit of possible $M_l$'s for the $l^{th}$ layer's neurons' mask. We find the suitable function for $\mu(l)$ by analyzing commonly used mathematical functions (\eg, cosine, logarithmic, cubic, \etc). We analyze the impact of these functions in Section~\ref{sec:ablation} and choose an exponential formulation (\eg, $~\alpha \cdot e^{-\beta \cdot l}$) for $\mu(l)$ since it produces the best possible outcome in terms of backdoor removal performance. Such formulation significantly reduces the mask search space, which leads to reduced runtime. Furthermore, the overall formulation for $M$ encourages relatively small changes to the original backdoor model's parameters. This helps us retain on-par clean test accuracy after backdoor removal, which is highly desirable for a defense technique. To this end, we aim to purify the backdoor model by optimizing for the best possible mask $\widehat{M}$ that suppresses backdoor-affected neurons and bolsters the neurons responsible for clean test accuracy. Here, $\widehat{M}$ should give us a purified model as, $f_{\theta_c} (.), \text{where} ~\theta_c = \theta\odot \widehat{M}$. Noteworthily, we do not change the bias as it may harm the classification accuracy. 
\begin{figure}[t]
  \centering
  \begin{subfigure}{0.32\linewidth}
    \includegraphics[width=1\linewidth]{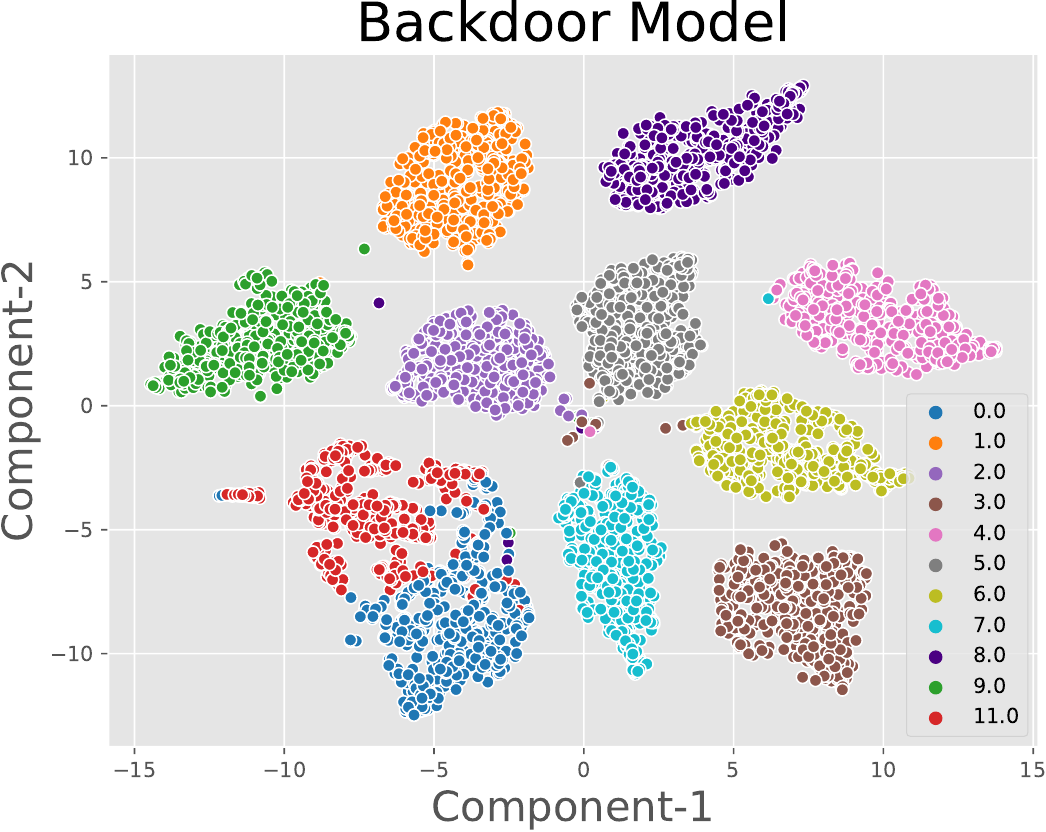}
    \caption{Backdoor Model}
    \label{fig:tsne_backdoor}
  \end{subfigure}
  \begin{subfigure}{0.32\linewidth}
    \includegraphics[width=1\linewidth]{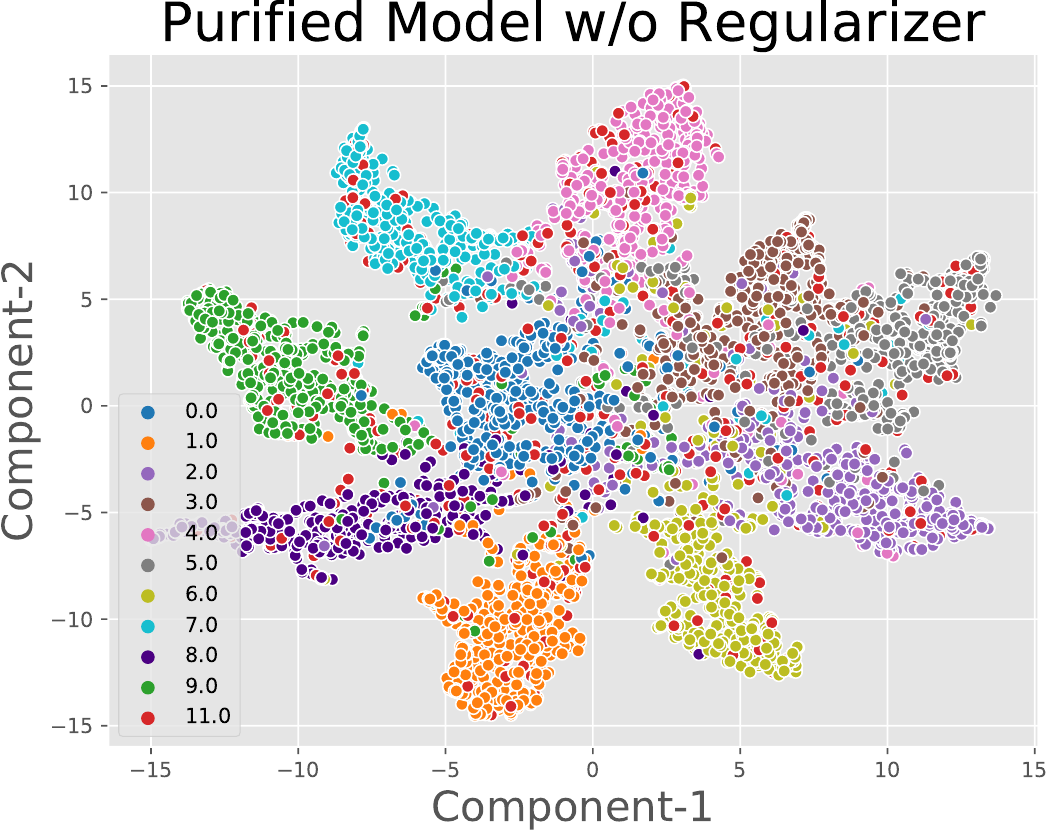}
    \caption{Purification w/o Reg.}
    \label{fig:wo_reg}
  \end{subfigure}
  \begin{subfigure}{0.32\linewidth}
    \includegraphics[width=1\linewidth]{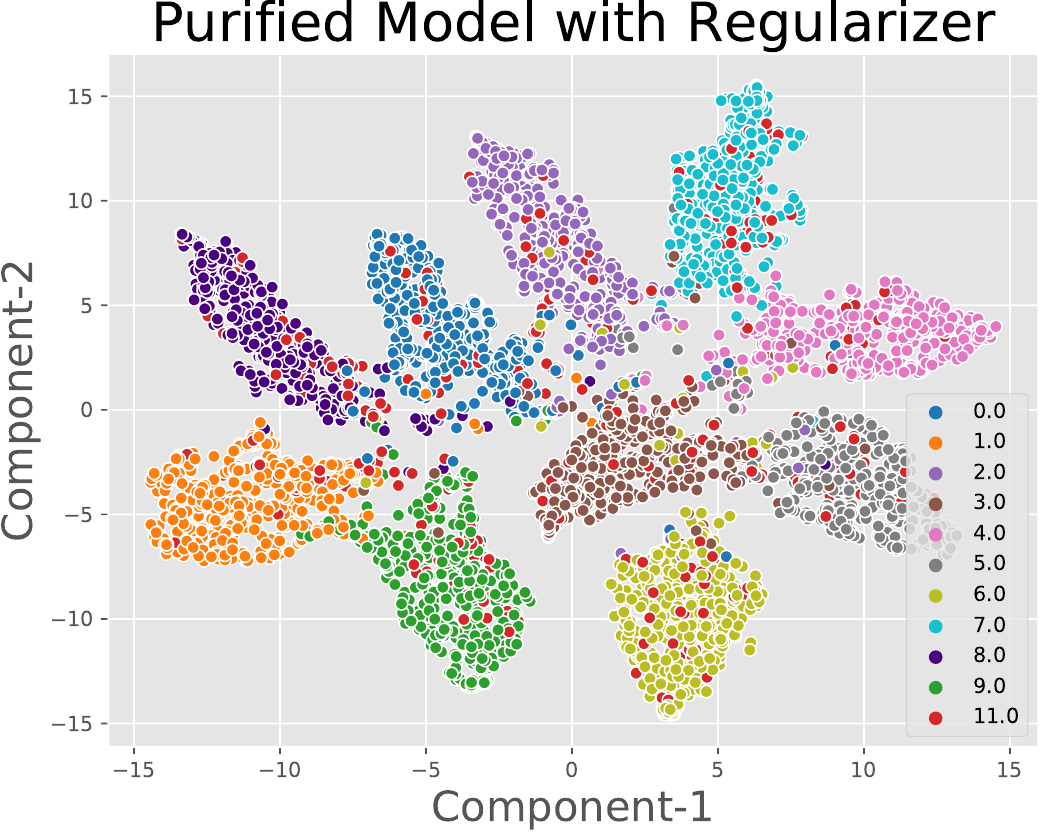}
    \caption{Purification with Reg.}
    \label{fig:with_reg}
  \end{subfigure}
  \caption{ \textbf{t-SNE visualization} of a backdoor model, where we show the ``\emph{poison cluster}'' with red color and label ``11''. Since the attack target class is ``0'', cluster ``0'' and the poison cluster sit closely with each other (Fig.~\ref{fig:tsne_backdoor}). After purification, the cluster should break, and all triggered samples should be classified according to their original label. In Fig.~\ref{fig:wo_reg}, we perform one-shot NFT without employing the regularizer. Due to the overfitting issue, the clean clusters lose their separability that can be established with \emph{Mask regularizer}, which tackles this issue (larger cluster gaps as compared to scenarios in Fig.~\ref{fig:wo_reg}) by keeping purified model parameters close to the original backdoor model (Fig.~\ref{fig:with_reg}). This, in turn, produces better clean test accuracy. For evaluation, we train a PreActResNet18~\cite{he2016identity} on CIFAR10 dataset with a poison rate of 10\%.}
  \label{fig:tSNE_main}
  \vspace{-3mm}
\end{figure}

\begin{table}[t]
\centering
 \caption{ Comparison of different defense methods for \textbf{CIFAR10 and ImageNet}. Average drop ($\downarrow$) indicates the \% changes in ASR/ACC compared to the baseline, \ie, ASR/ACC of \emph{No Defense}. A good defense should have a large \emph{ASR drop} with a small \emph{ACC drop}. Attacks are implemented with a poison rate of $10\%$.
 }
\scalebox{0.69}{
\begin{tabular}{c|c|cc|cc|cc|cc|cc|cc|cc}
\toprule
\multirow{2}{*}{Dataset} & Method & \multicolumn{2}{c|}{\begin{tabular}[c|]{@{}c@{}}No Defense\end{tabular}} & \multicolumn{2}{c|}{ANP} & \multicolumn{2}{c|}{I-BAU} & \multicolumn{2}{c|}{AWM} & \multicolumn{2}{c|}{FT-SAM} & \multicolumn{2}{c|}{RNP} & \multicolumn{2}{|c}{NFT (Ours)}\\ \cmidrule{2-16}
 & Attacks & ASR &ACC & ASR &ACC & ASR &ACC & ASR &ACC & ASR &ACC & ASR &ACC & ASR &ACC \\ \cmidrule{1-16}
\multirow{12}{*}{ CIFAR-10} 
  &  \emph{Benign} & 0 & 95.21 & 0 & 92.28 & 0 & 93.98&0&93.56&0&93.80 & 0 &  93.16 & 0 & \textbf{94.10}  \\
 &   Badnets & 100 & 92.96 & 4.87 & 85.92 & 2.84 & 85.96&5.72&87.85&4.34&86.17 & 2.75 & 88.46  & \textbf{1.74} &\textbf{90.82} \\
 &   Blend & 100 & 94.11  & 4.77 & 87.61 & 3.81 & 89.10&5.53&90.84&2.13& 88.93 & 0.91 & 91.53 &\textbf{ 0.31} & \textbf{93.17} \\
  &  Troj-one & 100 & 89.57& 3.78 & 82.18& 5.47 & {86.20} & 6.91&87.24&5.41&86.45 & 3.84 & 87.39 &\textbf{1.64} & \textbf{87.71} \\
  &  Troj-all & 100 & 88.33 & 3.91 & 81.95& 5.53 & 84.89 & 6.82&85.94&4.42&84.60 & 4.02 & 85.80 & \textbf{1.79} & \textbf{87.10} \\
 &   SIG & 100 & 88.64 & 1.04 & 81.92 & 0.37 & 83.60 &4.12&83.57&0.90&85.38&  0.51 & 86.46 & \textbf{ 0.12} & \textbf{87.16}  \\
  &  Dyn-one & 100 & 92.52 & 4.73 & 88.61 & 1.78 & 87.26 &7.48&\textbf{91.16} &3.35&88.41  & 8.61 & 90.05 & \textbf{1.37} & 90.81\\
 &  Dyn-all & 100 & 92.61 & 4.28 & 88.32 & 2.19 & 84.51 &7.30&89.74&2.46&87.72   & 10.57 & 90.28 & \textbf{1.42} & \textbf{91.53}\\
 &  CLB & 100 & 92.78 & \textbf{0.83} & 87.41 & 1.41 & 85.07 &5.78&86.70&1.89&84.18 & 6.12 & 90.38 & 1.04 & \textbf{90.37}\\
 &  CBA & 93.20 & 90.17 & 27.80 & 83.79 &45.11&85.63&36.12&85.05&38.81&85.58  & 17.72 & 86.40 & \textbf{21.60}&\textbf{87.97}\\
 & FBA & 100&90.78 &7.95 &82.90 &66.70&87.42&10.66&87.35&22.31&87.06 & 9.48 & 87.63 &\textbf{6.21}&\textbf{88.56}\\
 & WaNet&98.64&92.29&5.81&86.70&3.18&89.24&7.72&86.94&2.96&88.45 & 8.10 & \textbf{90.26} &\textbf{2.38}&89.65\\
 & ISSBA&99.80&92.78&6.76&85.42&\textbf{3.82}&89.20&12.48&90.03&4.57&89.59&  7.58 & 88.62 & 4.24&\textbf{90.18}\\
& LIRA & 99.25 & 92.15 & 7.34 & 87.41 & 4.51 & 89.61 & 6.13 & 88.50 & 3.86 & 89.22 & 11.83 & 87.59 &\textbf{1.53}&\textbf{90.57} \\ 
& BPPA & 99.70 & 93.82 & 9.94 & 90.23 & 10.46&90.57&9.94&90.68&10.60&90.88& 9.74 & 91.37&\textbf{5.04}&\textbf{91.78}\\
 \cmidrule{2-16} 
 & {Avg. Drop} & - &-  &  $92.63\downarrow$ & $5.94\downarrow$ &  $88.10\downarrow$ &  $4.66\downarrow$ &$91.21\downarrow$&$3.71\downarrow$&$92.61\downarrow$&$4.26\downarrow$& $92.06\downarrow$ & $2.95\downarrow$ & {\textbf{95.56}} $\downarrow$ &{\textbf{1.81}} $\downarrow$ \\ 
 \cmidrule{1-16}

\multirow{10}{*}{ ImageNet} 
&\emph{Benign}&0&77.06&0&73.52&0&71.85&0&74.21&0&71.63& 0 & 75.20 &0&\textbf{75.51}\\
&Badnets&99.24&74.53&5.91&69.37&6.31&66.28&\textbf{2.87}&69.46&4.18&69.44 & 7.58 & 70.49 &3.61&\textbf{70.96}\\
&Troj-one&99.21&74.02&7.63&69.15&7.73&67.14&5.74&69.35&2.86&70.62 & 2.94 & 72.17 &\textbf{3.16}&\textbf{72.37}\\
&Troj-all&97.58&74.45&9.18&69.86&7.54&68.20&6.02&69.64&3.27&69.85& 4.81 & 71.45 &\textbf{2.68}&\textbf{72.13}\\
&Blend&100&74.42&6.43&70.20&7.79&68.51&7.45&68.61&8.15&68.91& 5.69 & 70.24 &\textbf{3.83}&\textbf{71.52}\\
&SIG&94.66&74.69&\textbf{1.23}&69.82&4.28&66.08&5.37&70.02&3.47&69.74& 4.36 & 70.73&{2.94}&\textbf{72.36}\\
&CLB&95.08&74.14&6.71&69.19&4.37&66.41&7.64&69.70&3.50&69.32& 9.44 & 71.52&\textbf{3.05}&\textbf{72.25}\\
&{Dyn-one}&98.24&74.80&6.68&69.65&8.32&69.61&8.62&70.17&4.42&70.05& 12.56 & 70.39 &\textbf{2.62}&\textbf{71.91}\\
&{Dyn-all}&98.56&75.08&13.49&70.18&9.82&68.92&12.68&70.24&4.81&69.90& 14.18 & 69.47 &\textbf{3.77}&\textbf{71.62}\\
&{LIRA}&96.04&74.61&12.86&69.22&12.08&69.80&13.27&69.35&3.16& 12.31 & 70.50 &\textbf{71.38}&\textbf{2.62}&70.73\\
&
{WaNet}&97.60&74.48&6.34&68.34&5.67&67.23&6.31&70.02&\textbf{4.42}&66.82& 7.78 & 71.62 &4.71&\textbf{71.63}\\
&{ISSBA}&98.23&74.38&7.61&68.42&4.50&67.92&8.21&69.51&3.35&68.02& 9.74 & 70.81&\textbf{2.06}&\textbf{70.67} \\ 
\cmidrule{2-16}  
& {Avg. Drop} & - & -  & 90.08$\downarrow$ &  5.17$\downarrow$ & 88.90$\downarrow$ & 7.41 $\downarrow$ &$90.01\downarrow$&$4.72\downarrow$&$92.24\downarrow$& $5.61\downarrow$ & $89.37\downarrow$ & $3.66\downarrow$ & {\textbf{94.03}} $\downarrow$ &{\textbf{2.84}} $\downarrow$\\ 
\bottomrule
\end{tabular}}
\label{tab:main}
\end{table}

\begin{table}    
    \centering
    \caption{ Performance analysis for the \textbf{multi-label backdoor attack}~\cite{chen2023clean}. We choose 3 object detection datasets~\cite{everingham2010pascal,lin2014microsoft} and ML-decoder~\cite{ridnik2023ml} network architecture for this evaluation. Mean average precision (mAP) and ASR of the model, with and without defenses, have been shown.}
    \scalebox{0.82}{
    \begin{tabular}{l|cc|cc|cc|cc|cc|cc}
        \toprule
        \multirow{2}{*}{Dataset} & \multicolumn{2}{c|}{No defense}  & \multicolumn{2}{c|}{ANP} & \multicolumn{2}{c|}{AWM} & \multicolumn{2}{c|}{RNP} & \multicolumn{2}{c|}{FT-SAM} & \multicolumn{2}{|c}{NFT (Ours)} \\ \cmidrule{2-13}
         & ASR & mAP & ASR & mAP & ASR & mAP & ASR & mAP & ASR & mAP & ASR & mAP \\ \cmidrule{1-13}
        \textbf{VOC07} & 86.4 & 92.5 & 21.7 & 86.9 & 26.6 & 87.3 & 19.2 & 87.6 & 19.3 & 86.8 & \textbf{17.3} & \textbf{89.1} \\
        \textbf{VOC12} & 84.8 & 91.9 & 18.6 & 85.3 & 19.0 & 85.9 & \textbf{13.8} & 86.4 & 14.6 & 87.1 & 14.2 & \textbf{88.4} \\
        \textbf{MS-COCO} & 85.6 & 88.0 & 19.7 & 84.1 & 22.6 & 83.4 & 17.1 & 84.3 & 19.2 & 83.8 &  \textbf{16.6} & \textbf{85.8} \\
        \bottomrule
    \end{tabular}}
    \label{tab:multi_label}
\end{table}

\subsection{Sample Efficiency of NFT} \label{sec:sample_eff_NFT}

In this section, we discuss how careful modifications to the optimization scenario can make NFT highly sample-efficient. For this analysis, we first train a backdoor model (PreActResNet-18~\cite{he2016identity}) on a poisoned CIFAR10 dataset for 200 epochs. For poisoning the dataset, we use TrojanNet~\cite{liu2017trojaning} backdoor attack with a poison rate of 10\%. In Fig.~\ref{fig:tSNE_main}, we show t-SNE visualization of different class clusters obtained using the backdoor model. Instead of 10 clusters, we have an additional cluster (red color, labeled $``11"$) that sits closely with the original target class cluster (in this case, the target class is $``0"$). We name this cluster ``poison cluster'', whereas other clusters are ``clean clusters''. This cluster contains the embeddings of attacked or triggered samples from all other classes. The goal of any defense system is to break the formation of the poison cluster so that poison samples return to their original clusters. 

\noindent \textbf{One-Shot NFT.} Let us consider that there is only 1 sample (per class) available for the validation set $\mathbb{D}_\mathrm{val}$. Applying NFT with this $\mathbb{D}_\mathrm{val}$ forms the clusters shown in Fig.~\ref{fig:wo_reg}. Notice that the poison cluster breaks even with one-shot fine-tuning, indicating the backdoor's effect is removed successfully. However, since only one sample is available per class, the model easily overfits $\mathbb{D}_\mathrm{val}$, reducing margins between clean clusters. Such unwanted \emph{overfitting} issue negatively impacts the clean test accuracy. To combat this issue in scenarios where very few samples are available, we add a simple regularizer term as follows,
\begin{equation}\label{eq:aug_nft}
    \argmin_{M~\mid~m_l^{(i)}\in [\mu (l),1],~\forall l,i} \mathcal{L}^{\mathrm{mix}}({\theta \odot M},\mathbb{D}_\mathrm{val}) + \eta_c ||M_0-M||_1
\end{equation} 
where $M_0$ are the initial mask values (initialized as 1's) and $\eta_c$ is the regularizer coefficient. By minimizing the $\ell_1$-norm of the mask differences, we try to keep the purified model parameters ($\theta \odot M$) close to the original backdoor model parameters ($\theta \odot M_0$). We hope to preserve the original decision boundary between clean clusters by keeping these parameters as close as possible.  Note that the overfitting issue is not as prominent whenever we have a reasonably sized $\mathbb{D}_\mathrm{val}$ (\eg, 1\% of $\mathbb{D}_\mathrm{train}$). Therefore, we choose the value of $\eta_c$ to be $5e^{-4}/n_c$, which dynamically changes based on the number of samples available per class ($n_c$). As the number of samples increases, the impact of the regularizer reduces. \emph{Note that, Eq.~\eqref{eq:aug_nft} represents the final optimization function for our proposed method.}

\section{Experimental Results}\label{sec:exp_result}

\subsection{Evaluation Settings}


\noindent \textbf{Datasets.} We evaluate the proposed method through a range of experiments on two widely used datasets for backdoor attack study: {\bf{CIFAR10}}~\cite{krizhevsky2009learning} with 10 classes, {\bf{GTSRB}}~\cite{stallkamp2011german} with 43 classes. For the scalability test of our method, we also consider {\bf{Tiny-ImageNet}}~\cite{le2015tiny} with 100,000 images distributed among 200 classes and {\bf{ImageNet}}~\cite{deng2009imagenet} with 1.28M images distributed among 1000 classes. For multi-label clean-image backdoor attacks, we use object detection datasets \textbf{Pascal VOC}~\cite{everingham2010pascal}  and \textbf{MS-COCO}~\cite{lin2014microsoft}. {\bf{UCF-101}}~\cite{soomro2012ucf101} and \textbf{HMDB51}~\cite{kuehne2011hmdb} have been used for evaluating in action recognition task. The  \textbf{ModelNet}~\cite{wu20153d} dataset was also utilized to assess the performance of a 3D point cloud classifier. In addition to vision, we also consider attacks on natural language generation and use \textbf{WMT2014 En-De}~\cite{bojar-etal-2014-findings} machine translation and \textbf{OpenSubtitles2012}~\cite{tiedemann2012parallel}
dialogue generation datasets (Results are in the Supplementary).

\begin{table}[t]
    \centering
    \caption{Performance analysis for \textbf{action recognition tasks} where we choose 2 video datasets for evaluation. We consider a clean-label attack~\cite{zhao2020clean}, where we need to generate adversarial perturbations for each input frame.}
    \vspace{-1mm}        
    \scalebox{0.8}{
    \begin{tabular}{l|cc|cc|cc|cc|cc|cc}
        \toprule
        \multirow{2}{*}{Dataset} & \multicolumn{2}{c|}{No defense}  & \multicolumn{2}{c|}{I-BAU} & \multicolumn{2}{c|}{AWM} &  \multicolumn{2}{c|}{RNP} & \multicolumn{2}{c|}{FT-SAM} & \multicolumn{2}{c}{NFT (Ours)} \\ \cmidrule{2-13}
         & ASR & ACC & ASR & ACC & ASR & ACC & ASR & ACC & ASR & ACC & ASR & ACC \\ \cmidrule{1-13}
        UCF-101 & 81.3 & 75.6 & 20.4 & 70.6 & 20.8 & 70.1 & 17.0 & 70.3 & 15.9 & 71.6 &  \textbf{13.3} & \textbf{71.2}  \\
        HMDB-51 & 80.2 & 45.0 & 17.5 & \textbf{41.1} & 15.2 & 40.9 & 12.6 & 40.4 & 10.8 & 41.7 & \textbf{9.4} & 40.8  \\
        \bottomrule
    \end{tabular}}
    \vspace{-1mm}
    \label{tab:action_rec}
\end{table}

\noindent\textbf{Attacks Configurations.} Here, we first briefly overview the attack configurations on single-label image recognition datasets. We consider 14 state-of-the-art backdoor attacks: 1) Badnets~\cite{gu2019badnets}, 2) Blend attack~\cite{chen2017targeted}, 3 \& 4) TrojanNet (Troj-one \& Troj-all)~\cite{liu2017trojaning}, 5) Sinusoidal signal attack (SIG)~\cite{barni2019new}, 6 \& 7) Input-Aware Attack (Dyn-one and Dyn-all)~\cite{nguyen2020input}, 8) Clean-label attack (CLB) ~\cite{turner2019cleanlabel}, 9) Composite backdoor (CBA)~\cite{lin2020composite}, 10) Deep feature space attack (FBA)~\cite{cheng2021deep}, 11) Warping-based backdoor attack (WaNet)~\cite{nguyen2021wanet}, 12) Invisible triggers based backdoor attack (ISSBA)~\cite{li2021invisible}, 13) Imperceptible backdoor attack (LIRA)~\cite{doan2021lira}, and 14) 
Quantization and contrastive learning-based attack (BPPA)~\cite{wang2022bppattack}. In order to facilitate a fair comparison, we adopt trigger patterns and settings similar to those used in the original papers. Specifically, for both Troj-one and Dyn-one attacks, we set all triggered images to have the same target label (\ie, all2one), whereas, for Troj-all and Dyn-all attacks, we have uniformly distributed the target labels across all classes (\ie, all2all). Details on the hyper-parameters and overall training settings can be found in the Supplementary. We measure the success of an attack using two metrics: \emph{clean test accuracy (ACC)} defined as the percentage of clean samples that are classified to their original target label and {\em attack success rate (ASR)} defined as the percentage of poison test samples ($\hat{x}$) that are classified to the target label ($\hat{y}$).

\noindent\textbf{Defenses Configurations.} We compare our approach with 10 existing backdoor mitigation methods:  1) \textit{FT-SAM}~\cite{zhu2023enhancing}; 2) Adversarial Neural Pruning (\textit{ANP}) \cite{wu2021adversarial}; 3) Implicit Backdoor Adversarial Unlearning (\textit{I-BAU})~\cite{zeng2021adversarial}; 4) Adversarial Weight Masking (\textit{AWM})~\cite{chai2022one}; 5) Reconstructive Neuron Pruning (RNP)~\cite{li2023reconstructive}; 6) Fine-Pruning (\textit{FP})~\cite{liu2017neural}; 7) Mode Connectivity Repair (\textit{MCR})~\cite{zhao2020bridging}; 8) Neural Attention Distillation (\textit{NAD})~\cite{li2021neural}, 9) Causality-inspired Backdoor Defense (\textit{CBD})~\cite{zhang2023backdoor}), 10) Anti Backdoor Learning (\textit{ABL})~\cite{li2021anti}. In the main paper, we compare NFT with the first 5 defenses as they are more relevant, and the comparison with the rest of the methods is in \emph{Supplementary.} To apply NFT, we take $1\%$ clean validation data (set aside from the training set) and fine-tune the model for 100 epochs. An SGD optimizer has been employed with a learning rate of 0.05 and a momentum of 0.95. The rest of the experimental details for NFT and other defense methods are in the \emph{Supplementary}.

\subsection{Performance Comparison of NFT}
In this section, we compare the performance of NFT with other defenses in various scenarios: single-label (\ie, image classification), multi-label (\ie, object detection), video action recognition, and 3D point cloud.

\begin{table}[t]
    \centering
    \caption{Removal performance (\%) of NFT against backdoor attacks on \textbf{3D point cloud classifiers}. The attack methods~\cite{li2021pointba}, namely Poison-Label Backdoor Attack (PointPBA) with interaction trigger (PointPBA-I), PointPBA with orientation trigger (PointPBA-O), and Clean-Label Backdoor Attack (PointCBA) were considered, as well as the ``backdoor points'' based attack (3DPC-BA) outlined in prior work~\cite{xiang2021backdoor}.}
    \scalebox{0.8}{
    \begin{tabular}{l|cc|cc|cc|cc|cc|cc}
        \toprule
        \multirow{2}{*}{Attack} & \multicolumn{2}{c|}{No Defense} &  \multicolumn{2}{c|}{ANP} & \multicolumn{2}{c|}{AWM} &  \multicolumn{2}{c|}{RNP} & \multicolumn{2}{c|}{FT-SAM} & \multicolumn{2}{c}{NFT (Ours)} \\ \cmidrule{2-13}
         & ASR & ACC & ASR & ACC & ASR & ACC & ASR & ACC & ASR & ACC & ASR & ACC  \\ \cmidrule{1-13}
        PointBA-I & 98.6 & 89.1 & 13.6 & 82.6 & 15.4 & 83.9 & \textbf{8.1} & 84.0 & 8.8 & 84.5 & 9.6 & \textbf{85.7}  \\
        PointBA-O & 94.7 & 89.8 & 14.8 & 82.0 & 13.1 & 82.4 & 9.4 & 83.8 & 8.2 & 85.0 & \textbf{7.5} & \textbf{85.3} \\
        PointCBA & 66.0 & 88.7 & 21.2 & 83.3 & 21.5 & 83.8 & \textbf{18.6} & 84.6 & 20.3 & 84.7 & 19.4 & \textbf{86.1} \\
        3DPC-BA & 93.8 & 91.2 & 16.8 & 84.7 & 15.6 & 85.9 & 13.9 & 85.7 & 13.1 & 86.3 & \textbf{12.6} & \textbf{87.7}  \\
        \bottomrule
    \end{tabular}}
    \label{tab:3d_point_cloud}
\end{table} 

\noindent \textbf{Single-Label Backdoor attack.} In Table~\ref{tab:main}, we present the performance of different defenses for 2 widely used benchmarks. For CIFAR10, we consider \emph{4 label poisoning attacks: Badnets, Blend, Trojan, and Dynamic}. For all of these attacks, NFT shows significant performance improvement over other baseline methods. While ANP and AWM defenses work well for mild attacks with low poison rates (\eg, 5\%), the performance deteriorates for attacks with high poison rates (\eg, $\geq$ 10\%). It is observable that NFT performs well across all attack scenarios, \eg, obtaining a $99.69\%$ drop in ASR for blend attack while also performing well for clean data (only $0.94\%$ of ACC drop). For Trojan and Dynamic attacks, we consider two different versions of attacks based on label-mapping criteria (all2one and all2all). The drop in attack success rate shows the effectiveness of NFT against such attacks. Recently, \emph{small and imperceptible perturbations} as triggers have been developed in attacks (\eg, WaNet, LIRA, \etc) to fool the defense systems. While AWM generates perturbations as a proxy for these imperceptible triggers, NFT does a better job in this regard. For further validation of our proposed method, we use deep \emph{feature-based attacks} CBA and FBA. Both of these attacks manipulate features to insert backdoor behavior. Overall, we achieve an average drop of $95.56\%$ in ASR while sacrificing an ACC of $1.81\%$. For the scalability test, we consider a large and widely used dataset in vision tasks, ImageNet. In consistency with other datasets, NFT also obtains SOTA performance in this particular dataset. Due to page constraints, we move \emph{the performance comparison for GTSRB and Tiny-ImageNet to the Supplementary.}

\noindent \textbf{Multi-label Backdoor Attack.}
In Table~\ref{tab:multi_label}, we also evaluate our proposed method on multi-label clean-image backdoor attack~\cite{chen2023clean}. In general, we put a trigger on the image and change the corresponding ground truth of that image. However, a certain combination of objects has been used as a trigger pattern. For example, if a combination of car, person, and truck is present in the image, it will fool the model into misclassifying it. Table~\ref{tab:multi_label} shows that our proposed method surpasses other defense strategies concerning both ASR and mAP metrics. Notably, ANP and AWM, which rely on adversarial search limited applicability in multi-label scenarios. This limitation arises from the less accurate process of approximating triggers for object detection. Conversely, FT-SAM's optimization driven by sharpness proves effective in removing backdoors, yet it achieves a lower mAP post-purification. This outcome is not ideal since the goal is to eliminate backdoors without significantly compromising clean accuracy.

\begin{figure}[t]
\begin{minipage}{.475\textwidth}
  \centering
    {\includegraphics[width=0.465\textwidth]{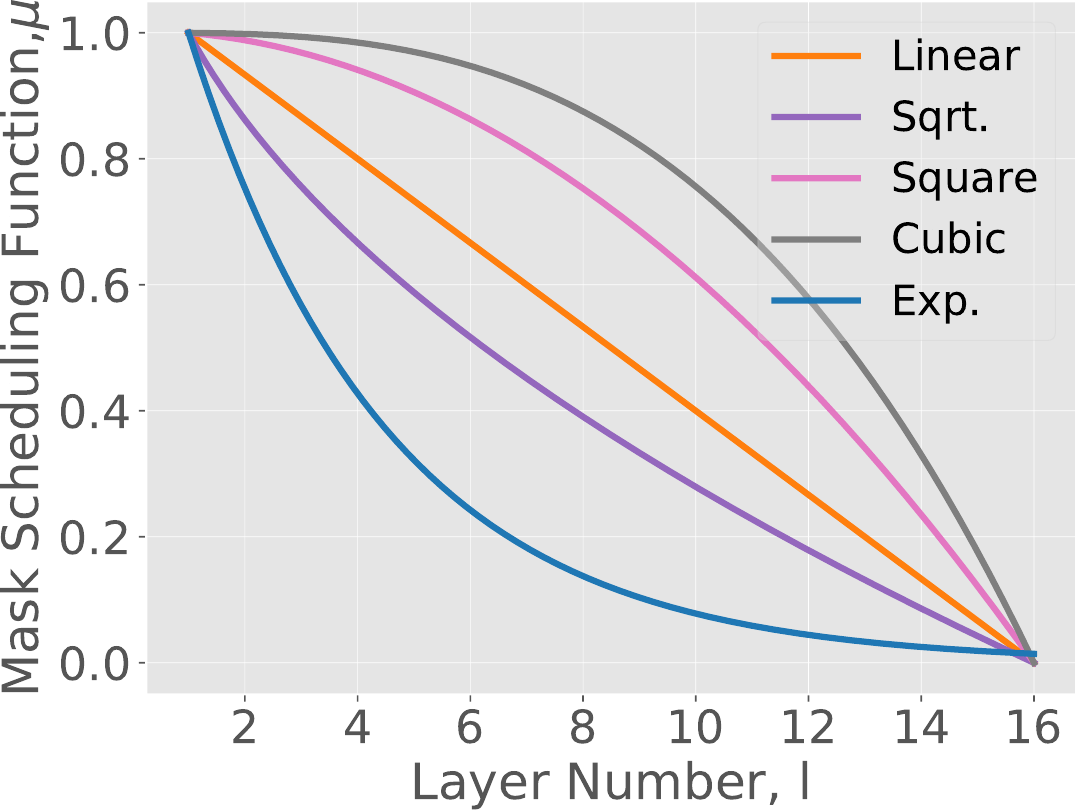}}\hfill
    {\includegraphics[width=0.465\textwidth]{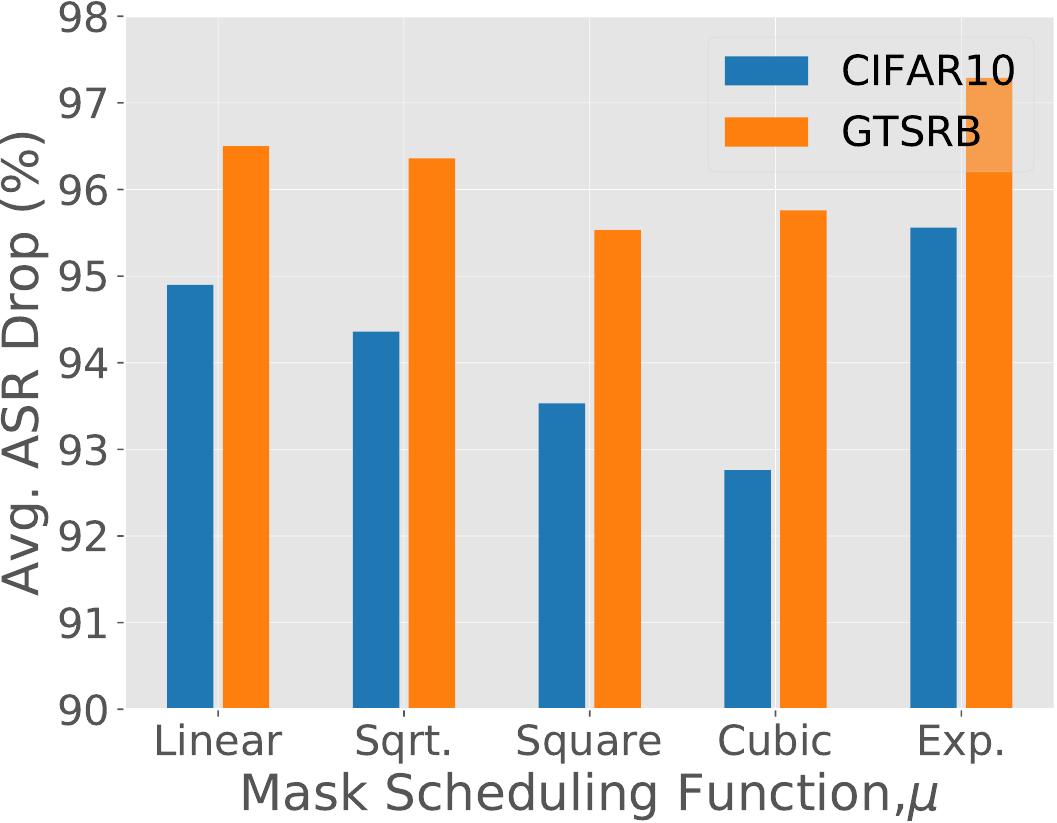}}\hfill
    \caption{ Ablation with different \textbf{Mask Scheduling Function ($\mu$)}. }
  \label{fig:math_func}
\end{minipage}
\hfill
\begin{minipage}{.475\textwidth}
  \centering
    {\includegraphics[width=0.465\textwidth]{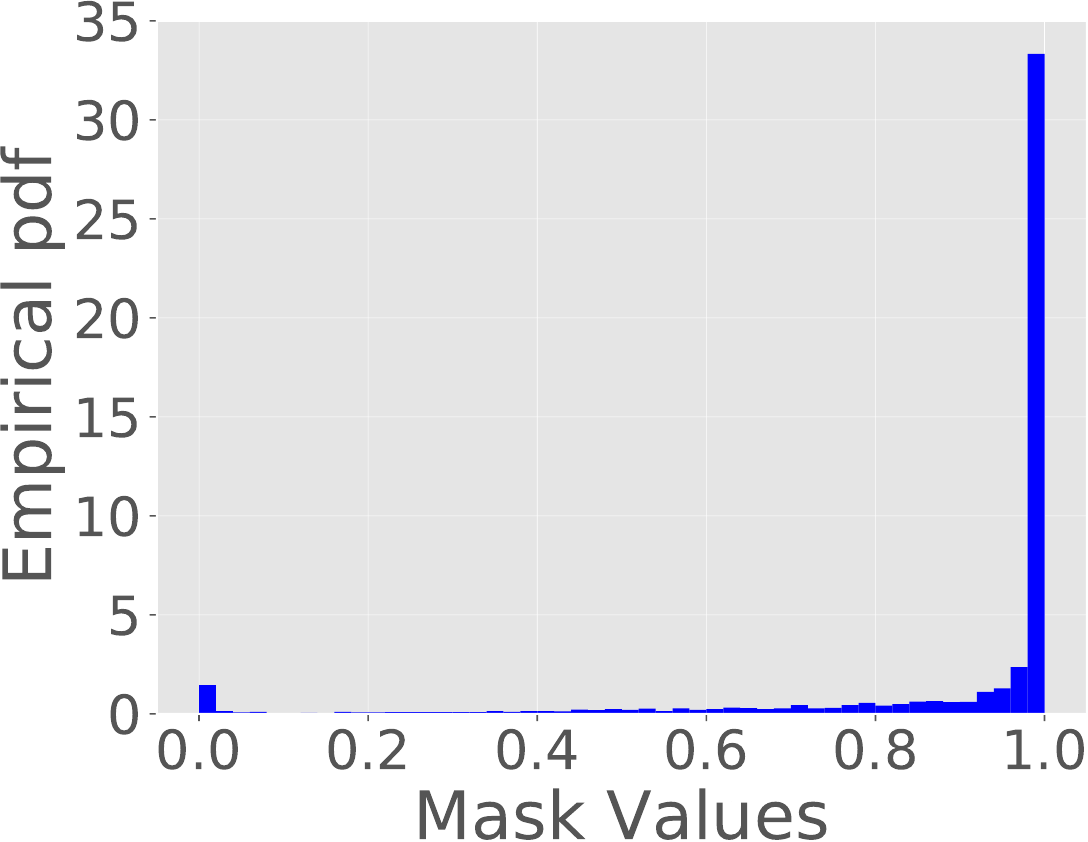}}\hfill
    {\includegraphics[width=0.465\textwidth]{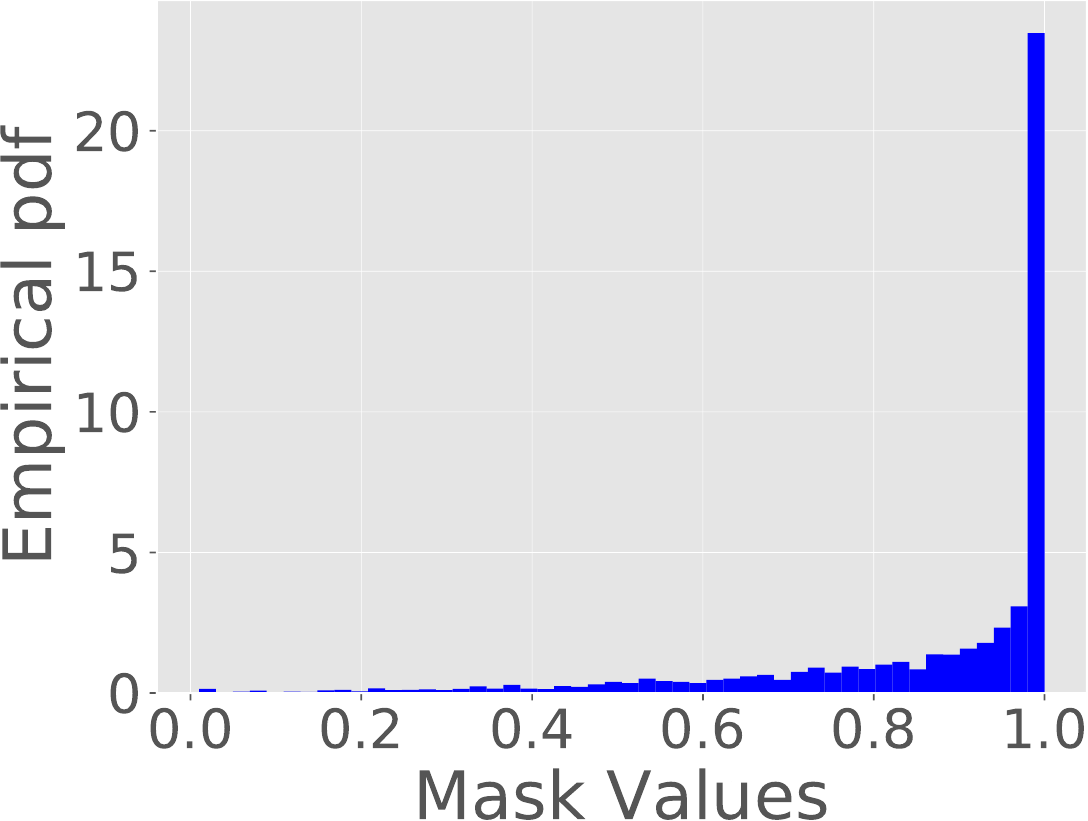}}\hfill
    \caption{ \textbf{Mask Distribution} of AWM (left) and NFT (right).}
  \label{fig:mask_dist}
\end{minipage}
\vspace{-4mm}
\end{figure}

\noindent \textbf{Action Recognition Model.} We further consider attacks on action recognition models; results are reported in Table~\ref{tab:action_rec}. We use two widely used datasets, UCF-101~\cite{soomro2012ucf101} and HMDB51~\cite{kuehne2011hmdb}, with a CNN+LSTM network architecture. An ImageNet pre-trained ResNet50 network has been used for the CNN, and a sequential input-based Long Short Term Memory (LSTM)~\cite{sherstinsky2020fundamentals} network has been put on top of it. We subsample the input video by keeping one out of every 5 frames and use a fixed frame resolution of $224 \times 224$. We choose a trigger size of $20\times20$. Following~\cite{zhao2020clean}, we create the required perturbation for clean-label attack by running projected gradient descent (PGD)~\cite{madry2017towards} for 2000 steps with a perturbation norm of $\epsilon=16$. Note that our proposed augmentation strategies for image classification are directly applicable to video recognition. During training, we keep 5\% samples from each class to use them later as the clean validation set. Table~\ref{tab:action_rec} shows that NFT outperforms other defenses by a significant margin. In the case of a high number of classes and multiple image frames in the same input, it is a challenging task to optimize for the proper trigger pattern through the adversarial search described in I-BAU and AWM. Without a good approximation of the trigger, these methods seem to underperform in most cases.     

\begin{table}[t]
\begin{minipage}{0.48\textwidth}
    \centering
    \caption{  \textbf{Average runtime} of different defense methods against all 14 attacks. An NVIDIA RTX3090 GPU is used. }  
     \scalebox{0.65}{
    \begin{tabular}{c|cccccc}
    \toprule
     Method & ANP & I-BAU & AWM  & FT-SAM & RNP & NFT (Ours) \\
    \midrule
     Runtime (sec.) &  118.1 & 92.5 & 112.5 & 98.7 & 102.6 & \textbf{28.4} \\
    \bottomrule
    \end{tabular}}
    \label{tab:runtime_analysis}
\end{minipage}
\hfill
\begin{minipage}{0.48\textwidth}
    \centering
    \caption{\textbf{Sensitivity analysis of $\alpha$ and $\beta$} for LIRA on CIFAR10. }
     \scalebox{0.625}{
    \begin{tabular}{c|c|c|c|c|c|c|c|c|c}
    \toprule
     $\alpha$ & 0.9 & 0.9 & 0.9 & 0.8 & 0.8 & 0.8 & 0.7 & 0.7 & 0.7   \\
     \midrule
     $\beta$ & 0.25 & 0.50 & 0.75 & 0.25 & 0.50 & 0.75 & 0.25 & 0.50 & 0.75 \\
     \midrule
     ASR & 3.65 & 3.42 & 4.87 & 2.73 & \textbf{1.53} & 3.76 & 4.18 & 4.92 & 5.23    \\
     ACC & 88.52 & 88.34 & 89.12 & 89.97 & \textbf{90.57} & 89.34 & 87.64 & 87.78 & 88.10\\   
    \bottomrule
    \end{tabular}}
    \label{tab:sensitivity_alpha}
\end{minipage}
\vspace{-3mm}
\end{table}

\noindent \textbf{3D Point Cloud.} In this phase of our work, we assess NFT's resilience against attacks on 3D point cloud classifiers~\cite{li2021pointba,xiang2021backdoor}. To evaluate, we utilize the ModelNet dataset~\cite{wu20153d} and the PointNet++ architecture ~\cite{qi2017pointnet++}. The performance comparison of NFT and other defense methods is outlined in Table~\ref{tab:3d_point_cloud}. NFT outperforms other defenses due to its unique formulations of the objective function.

\subsection{Ablation Study}\label{sec:ablation}
For all ablation studies, we consider the CIFAR10 dataset.

\noindent \textbf{Runtime Analysis.} For runtime analysis, we present the training time for different defenses in Table~\ref{tab:runtime_analysis}. ANP and AWM both employ computationally expensive adversarial search procedures to prune neurons, which makes them almost 4x slower than our method. However, NFT offers a computationally less expensive defense with SOTA performance in major benchmarks. 

\noindent\textbf{Choice of Scheduling Function, $\mu$.} For choosing the suitable function for $\mu$, we conduct a detailed study with commonly used mathematical functions. Note that, we only consider a family of functions that decreases over the depth of the network (shown in Fig.~\ref{fig:math_func}). This allows more variations for deeper layer weights, making sense as they are more responsible for DNN decision-making. In our work, we choose an exponential formulation for $\mu$ as it offers superior performance. We also perform sensitivity analysis of scheduling parameters $\alpha$ and $\beta$ in Table~\ref{tab:sensitivity_alpha}. In Fig.~\ref{fig:mask_dist}, we show the generated mask distributions of AWM and NFT. Compared to AWM, NFT produces more uniformly distributed masks that seem helpful for backdoor purification. We show the impact of $\mu$ in Table~\ref{tab:val_size}.

\noindent\textbf{Nature of Optimization.} In Table~\ref{tab:val_size}, we present the performance of SOTA techniques under \emph{different validation sizes}. Even with 10 samples (single-shot), NFT performs reasonably well and offers better performance as compared to AWM. This again shows that the trigger generation process is less accurate and effective for a very small validation set. We also show the \emph{effect of the proposed mask regularizer} that indirectly controls the change in weights for better ACC. Although AWM employs a similar $\ell_1$ regularizer for masks, our proposed regularizer is more intuitive and specifically designed for better ACC preservation. While AWM encourages sparse solutions for $M$ (shown in the left subfigure of Fig.~\ref{fig:mask_dist}), it helps with ASR but heavily compromises ACC. We also show the performance of NFT without augmentations. 

\begin{table}
    \centering
    \caption{ Purification performance (\%) for \textbf{various validation data sizes}. NFT performs reasonably well even with as few as 10 samples, \ie, one sample (shot) per class for CIFAR10. We also show the impact of the \textbf{ mask regularizer}, \textbf{mask scheduling function $\mu$}, and \textbf{augmentations} on performance, which resonates with Fig.~\ref{fig:tSNE_main}. Mask regularizer has the most impact on the clean test accuracy (around 7\% worse without the regularizer). Without strong augmentations, we have a better ACC with a slightly worse ASR (around 6\% drop). }
    \scalebox{0.85}{
    \begin{tabular}{c|cc|cc|cc|cc|cc|cc}
    \toprule
         Attack & \multicolumn{4}{c|}{Dynamic}  & \multicolumn{4}{c|}{WaNet} & \multicolumn{4}{c}{LIRA}  \\
         \midrule
         Samples & \multicolumn{2}{c|}{10}  & \multicolumn{2}{c|}{100} & \multicolumn{2}{c|}{10}  & \multicolumn{2}{c|}{100} & \multicolumn{2}{c|}{10} & \multicolumn{2}{c}{100} \\
         \midrule
         Method & ASR & ACC& ASR & ACC& ASR & ACC& ASR & ACC& ASR & ACC& ASR & ACC\\
         \midrule
         \emph{No Defense} &  100 & 92.52 &  100 & 92.52  & 98.64 & 92.29  &98.64 & 92.29 & 99.25 & 92.15 & 99.25 & 92.15 \\
         AWM  & 86.74 & 55.73 & 9.16 & 85.33  & 83.01 & 62.21 & 7.23 & 84.38 & 91.45 & 66.64 & 10.83 & 85.87 \\
         FT-SAM  & 8.35 & 73.49 & 5.72 & 84.70  & 9.35 & 75.98 & 5.56 & 86.63 & 11.83 & 72.40 & 4.85 & 88.82  \\
         \midrule
         NFT w/o Reg.  & 5.67 & 76.74 & \textbf{1.36} & 82.21  & \textbf{4.18} & 76.72 & 3.02 & 83.31 & \textbf{4.83} & 74.58 & 2.32 & 83.61  \\
         NFT w/o Aug. & 11.91 & \textbf{81.86} & 10.59 & 89.53  & {10.36} & 83.10 & 7.81 & \textbf{89.68} & {12.23} & 81.05 & {9.16} & 88.74 \\
         NFT w/o $\mu(l)$ & 5.11 & 80.32 & 3.04 & 88.58  & {5.85} & 82.46 & 4.64 & 88.02 & {6.48} & 81.94 & {4.33} & 88.75 \\
         NFT  & \textbf{4.83} & 80.51 & 1.72 & \textbf{90.08}  & {4.41} & \textbf{83.58} & \textbf{2.96} & 89.15 & {5.18} & \textbf{82.72} & \textbf{2.04} & \textbf{89.34} \\
         \bottomrule
    \end{tabular}}
    \label{tab:val_size}
\end{table}

\noindent\textbf{Label Correction Rate.} In the standard backdoor removal metric, it is sufficient for backdoored images to be classified as a non-target class (any class other than $\hat{y}$). However, we also consider another metric, label correction rate (LCR), for quantifying the success of a defense. \emph{We define LCR as the percentage of poisoned samples correctly classified to their original classes}. Any method with the highest value of LCR is considered to be the best defense method. For this evaluation, we use CIFAR10 dataset and 6 backdoor attacks. Initially, the correction rate is 0\% with no defense as the ASR is close to 100\%. Table~\ref{tab:correction_rate} shows that NFT obtains better performance in terms of LCR.

\begin{table}[t]
    \centering
    \caption{ \textbf{Label Correction Rate} (\%) for different defense techniques, defined as the percentage of backdoor samples that are correctly classified to their original ground truth label.}  
    \scalebox{0.85}{
    \begin{tabular}{c|c|c|c|c|c|c}
    \toprule
         Defense & Badnets&Trojan&Blend&SIG&CLB & WaNet \\ \midrule
         No Defense& 0&0&0&0&0 & 0\\
         \midrule
         ANP &84.74&80.52&81.38&53.35&82.72 & 80.23\\
         I-BAU &78.41&77.12&77.56&39.46&78.07 & 80.65 \\
         AWM&79.37&78.24&79.81&44.51&79.86 & 79.18\\
         FT-SAM &85.56&80.69&84.49&\textbf{57.64}&82.04 & 83.62\\\midrule
         NFT (Ours)&\textbf{86.82}&\textbf{81.15}&\textbf{85.61}&{55.18}&\textbf{86.23}&\textbf{85.70}\\
         \bottomrule
    \end{tabular}}
    \label{tab:correction_rate}
\end{table}

\section{Conclusion}
We proposed a backdoor purification framework, NFT, utilizing an augmentation-based neural mask fine-tuning approach. NFT can change the backdoor model weights in a computationally efficient manner while ensuring SOTA purification performance. Our proposed method showed that the addition of MixUp during fine-tuning replaces the need for a computationally expensive trigger synthesis process. Furthermore, we proposed a novel mask regularizer that helps us preserve the cluster separability of the original backdoor model. By preserving this separability, the proposed method offers better clean test accuracy compared to SOTA methods. Furthermore, we suggested using a mask scheduling function that reduces the mask search space and improves the computational efficiency further. Our extensive experiments on 5 different tasks validate the efficiency and efficacy of the proposed backdoor purification method. We also conducted a detailed ablation study to explain the reasoning behind our design choices.



\section*{Acknowledgements}
This work was supported in part by the National Science Foundation under Grant ECCS-1810256 and CMMI-2246672.

\bibliographystyle{splncs04}
\bibliography{egbib}

\appendix
\renewcommand{\thesection}{Appendix \Alph{section}}

\section*{Overview}

The overview of our supplementary is as follows:

\begin{itemize}
    \item We provide proof for Theorem~1 in Section~\ref{sec:theorem}.
    \item Section~\ref{sec:experiments} contains the experimental setting of different backdoor attacks, NFT, and other baselines. 
    \item Section~\ref{sec:add_exp_results} contains the additional experimental results where we present the comparison for \emph{GTSRB and Tiny-ImageNet} in Section~\ref{sec:gtsrb_and_tiny}, results for \emph{natural language generation tasks} in Section~\ref{sec:nlg}, and comparison with \emph{additional SOTA defense methods} in Section~\ref{sec:additional_defense}. In our work, we use MixUp as data augmentation. However, we show the performance of NFT with \emph{other popular augmentation strategies} in Section~\ref{sec:common_aug}.   
    \item  We present more ablation study in Section~\ref{sec:more_ablation} where we show the performance of \textit{Adaptive Attacks}, \emph{One-Shot  NFT for the other 3 datasets}, \emph{impact of clean validation data size, Impact of $\eta_c$, Layerwise Mask Heamaps, augmented defenses, \etc}. An ablation study with different poison rates has also been presented.
\end{itemize}

\section{Theoretical Justifications.}\label{sec:theorem}

\noindent \textbf{Proof of Theorem 1.} For a fully-connected neural network (NN) with logistic loss \(\ell (y, f_\theta(x)) = \log(1+\exp{(f_\theta(x))}) - y f_\theta(x)\) with \(y \in \{0,1\}\), it can be shown that $\mathcal{L}^{\mathrm{mix}}(\theta,\mathbb{D}_\mathrm{val} )$ is an upper-bound of the second order Taylor series expansion of the ideal loss $\mathcal{L}^{\mathrm{ideal}} (\theta, \mathbb{D}_\mathrm{val})$. With the nonlinearity $\sigma$ for ReLU and max-pooling in NN, the function $f_\theta$ satisfies that $f_\theta(x) = \nabla f_\theta(x)^T x$ and $\nabla^2 f_\theta(x) = 0$ almost everywhere, where the gradient is taken with respect to the input $x$. 

We first rewrite the $\mathcal{L}^{\mathrm{ideal}} (\theta, \mathbb{D}_\mathrm{val})$ using Taylor series approximation. The second-order Taylor expansion of $\ell(y, f_\theta(x + \delta))$ is given by,
\begin{equation*}
    \ell(y, f_\theta(x + \delta)) = \ell(y, f_\theta(x)) + (g(f_\theta(x)) -y)(f_\theta(\delta)) + \frac{1}{2}g(f_\theta(x))(1-g(f_\theta(x)))(f_\theta(\delta))^2,
\end{equation*} 
where \(g(x) = \frac{e^x}{1+e^x}\) is the logistic function. 

Now using $f_\theta(\delta) = \nabla f_\theta (\delta)^T \delta \leq ||\nabla f_\theta (\delta)||_2 \cdot ||\delta||_2$, we get
\begin{equation*}
    \begin{aligned}
    \ell(y, f_\theta(x + \delta)) 
    \leq &~\ell(y, f_\theta(x)) + ||\delta||_2 \cdot |(g(f_\theta(x)) -y)|\cdot||\nabla f_\theta (\delta)||_2 \\
    &+ \frac{||\delta||^2_2}{2}|g(f_\theta(x))(1-g(f_\theta(x)))| \cdot||\nabla f_\theta (\delta)||^2_2     
\end{aligned}
\end{equation*}

Notice that the goal of ideal loss $\mathcal{L}^{\mathrm{ideal}} (\theta, \mathbb{D}_\mathrm{val})$ is to refine the model such that the model predicts $y$ for input $x$ or $x+\delta$, implying that the impact of model's gradient corresponding to $\delta$, $\nabla f_\theta(\delta)$, is sufficiently less than the model's gradient corresponding to $x$, $\nabla f_\theta(x)$, \ie, $\nabla f_\theta(\delta) \leq \nabla f_\theta(x)$. Therefore,

\begin{equation}\label{eq:idealT}
    \begin{aligned}
    \ell(y, f_\theta(x + \delta)) \leq &~\ell(y, f_\theta(x)) + ||\delta||_2 \cdot |(g(f_\theta(x)) -y)|\cdot||\nabla f_\theta (x)||_2 \\
    &+ \frac{||\delta||^2_2}{2}|g(f_\theta(x))(1-g(f_\theta(x)))| \cdot||\nabla f_\theta (x)||^2_2     
\end{aligned}
\end{equation}

Based on the MixUp related analysis in prior works~\cite{carratino2022mixup,zhang2020does}, the following can be derived for \(\mathcal{L}^{\mathrm{mix}}(\theta,\mathbb{D}_\mathrm{val} )\) using the second-order Taylor series expansion,

\noindent\textbf{Lemma 1.}
\emph{Assuming $f_\theta(x) = \nabla f_\theta(x)^T x$ and $\nabla^2 f_\theta(x) = 0$ (which are satisfied by ReLU and max-pooling activation functions), \(\mathcal{L}^{\mathrm{mix}}(\theta,\mathbb{D}_\mathrm{val} )\) can be expressed as,
\begin{equation}
    \mathcal{L}^{\mathrm{mix}}(\theta,\mathbb{D}_\mathrm{val} ) = \mathcal{L}(\theta,\mathbb{D}_\mathrm{val} ) + \mathcal{R}_1(\theta, \mathbb{D}_\mathrm{val}) +\mathcal{R}_2(\theta, \mathbb{D}_\mathrm{val})
\end{equation}  
where,
\[
\mathcal{R}_1(\theta, \mathbb{D}_\mathrm{val}) \geq \frac{Rc_x\E_\lambda[(1-\lambda)]\sqrt{d}}{N_\mathrm{val}} \sum_{i=1}^{N_\mathrm{val}}|g(f_\theta(x_i)) - y_i |\cdot||\nabla f_\theta(x_i)||_2
\]
\[
\mathcal{R}_2(\theta, \mathbb{D}_\mathrm{val}) \geq \frac{R^2c_x^2\E_\lambda[(1-\lambda)]^2{d}}{2N_\mathrm{val}} \sum_{i=1}^{N_\mathrm{val}}|g(f_\theta(x_i))(1 - g(f_\theta(x_i)))|\cdot||\nabla f_\theta(x_i)||_2^2,
\]
where \(R=\min_{i\in [N_\mathrm{val}]}\langle\nabla f_\theta(x_i), x_i\rangle/||\nabla f_\theta(x_i)||\cdot||x_i||\) and \(c_x > 0\) is a constant. 
}

By comparing $\ell(y, f_\theta(x + \delta))$ and $\mathcal{L}^{\mathrm{mix}}(\theta,\mathbb{D}_\mathrm{val} )$ for a fully connected NN, we can prove the following.

\noindent\textbf{Theorem 1.}
\emph{Suppose that $f_\theta(x) = \nabla f_\theta(x)^T x$, $\nabla^2 f_\theta(x) = 0$ and there exists a constant $c_x > 0$ such that $\|x_i\|_2 \geq c_x \sqrt{d}$ for all $i \in \{1, \ldots, N_{\mathrm{val}}\}$. Then, for any $f_\theta$, we have
\[
\mathcal{L}^{\mathrm{mix}}(\theta,\mathbb{D}_\mathrm{val}) \geq \frac{1}{N_{\mathrm{val}}} \sum_{i=1}^{N_{\mathrm{val}}} \ell\left(y_i, f_\theta(x_i+\varepsilon_i)\right) \geq \frac{1}{N_{\mathrm{val}}} \sum_{i=1}^{N_{\mathrm{val}}} \ell\left(y_i, f_\theta(x_i+\varepsilon)\right)
\]
where $\varepsilon_i = R_i c_x \E_{\lambda \sim \mathcal{D}_{\lambda}}[(1 - \lambda)]\sqrt{d}$ with $R_i = \langle\nabla f_\theta(x_i), x_i\rangle/||\nabla f_\theta(x_i)||\cdot||x_i||$ and $\varepsilon = \min\{\varepsilon_i\}$.
}

\begin{proof}
From Lemma 1, we get
\begin{equation*}
    \begin{aligned}
      \mathcal{L}^{\mathrm{mix}}(\theta,\mathbb{D}_\mathrm{val}) \geq &\mathcal{L}(\theta,\mathbb{D}_\mathrm{val}) + \frac{Rc_x\E_\lambda[(1-\lambda)]\sqrt{d}}{N_\mathrm{val}} \sum_{i=1}^{N_\mathrm{val}}|g(f_\theta(x_i)) - y_i |\cdot||\nabla f_\theta(x_i)||_2  \\
      &+ \frac{R^2c_x^2\E_\lambda[(1-\lambda)]^2{d}}{2N_\mathrm{val}} \sum_{i=1}^{N_\mathrm{val}}|g(f_\theta(x_i))(1 - g(f_\theta(x_i)))|\cdot||\nabla f_\theta(x_i)||_2^2 \\
      \overset{{(*)}}{\geq} &  \frac{1}{N_{\mathrm{val}}}\sum_{i=1}^{N_{\mathrm{val}}}\ell(y_i, f_\theta(x_i + \varepsilon))=\mathcal{L}^{\mathrm{ideal}} (\theta, \mathbb{D}_\mathrm{val}) 
    \end{aligned}
\end{equation*}
where step $(*)$ follows directly using Eq.~\eqref{eq:idealT} and $\varepsilon = R c_x \E_{\lambda \sim \mathcal{D}_{\lambda}}[(1 - \lambda)]\sqrt{d}$.
\end{proof}

Theorem 1 implies that as long as \(||\delta||_2 \leq \varepsilon\) holds, the MixUp loss \(\mathcal{L}^{\mathrm{mix}}(\theta,\mathbb{D}_\mathrm{val})\) can be considered as an upper-bound of \(\mathcal{L}^{\mathrm{ideal}}(\theta,\mathbb{D}_\mathrm{val})\). Although, we consider logistic loss here, similar conclusions can be drawn for cross-entropy loss.

\section{Experimental Settings}~\label{sec:experiments}
The CIFAR-10~\cite{krizhevsky2009learning} dataset consists of $60,000$ color images, which are classified into 10 classes. There are $50,000$ training images and $10,000$ test images for each class. GTSRB~\cite{stallkamp2011german} is also an image classification dataset. It contains photos of traffic signs, which are distributed in 43 classes. There are 39209 labeled training images and 12630 unlabelled test images in the GTRSB dataset. We rescale the GTSRB images to $32\times32$. Training hyperparameters details can be found in Table~\ref{tab:hyperparameters}-\ref{tab:hyperparameters_tiny_Imagnet}. We use NVIDIA RTX 3090 GPU for all experiments.

\begin{table}
\centering
\caption{Training \textbf{Hyper-Parameters} for CIFAR10 and GTSRB}

\scalebox{0.8}{
\begin{tabular}{c|c}
\toprule
\textbf{Hyper Parameters} & \textbf{Values} \\
\midrule
Image Size & $32\times32$ \\
Initial Learning Rate & $5e^{-2}$ \\
Momentum & 0.9 \\
Weight Decay & $5e^{-4}$ \\
Normalization (CIFAR10) &  Mean - [0.4914, 0.4822, 0.4465],  \hspace{2mm}\\
&Std. dev. - [0.2023, 0.1994, 0.2010] \\
Normalization (GTSRB) & None \\
Batch Size & 128 \\
Number of Training Epochs & 100 \\
\bottomrule
\end{tabular}}

\label{tab:hyperparameters}
\end{table}

\begin{table}
\centering
\caption{Training \textbf{Hyper-Parameters} for Tiny-ImageNet/ ImageNet. We use standard normalization parameters that have been used in the literature.}

\scalebox{0.85}{
\begin{tabular}{c|c}
\toprule
\textbf{Hyper Parameters} & \textbf{Values} \\
\midrule
Image Size & $64\times64$ and $224\times224$ \\
Initial Learning Rate & $1e^{-2}/1e^{-3}$ \\
Momentum & 0.9 \\
Weight Decay & $5e^{-4}$ \\
Normalization (Tiny-ImageNet) & Standard \\
Normalization (ImageNet) & Standard \\
Batch Size & 128/32 \\
Number of Training Epochs & 10/2 \\
\bottomrule
\end{tabular}}
\label{tab:hyperparameters_tiny_Imagnet}
\end{table}

\subsection{Attack Implementation Details}
Following our attack model, we create the triggered input as,
\(\hat{x}_i =  x_i +\delta,\) 
  where $\delta \in \mathbb{R}^d$ represents trigger pattern and the target label $\hat{y}_i \neq y_i$ (set by the adversary).
 Depending on the type of trigger, poison rate ($|\mathbb{D}'_{\mathsf{train}}|/|\mathbb{D}_{\mathsf{train}}|$) and label mapping ($\hat{y}_{i} \rightarrow y_i$), one can formulate the different type of backdoor attacks. In our work, we create 14 different backdoor attacks based on the trigger types, label-poisoning type, label mapping type, \etc For most types of attacks, we use a poison rate (ratio of poison data to training data) of 10\%. The details of the attacks are given below:

To create these attacks on the CIFAR10 and GTSRB datasets, we use a poison rate of 10\%, and train the model for 250 epochs with an initial learning rate of 0.01. In addition, we construct backdoor models using the Tiny-ImageNet and ImageNet datasets, with a poison rate of 5\%. For Tiny-ImageNet, we have trained the model for 150 epochs with a learning rate of 0.005, and a decay rate of 0.1/60 epochs.

\begin{table}[t]
\centering
 \caption{ Performance of NFT on a dataset with a large number of classes, \textbf{Tiny-ImageNet}. We employ ResNet34 architecture here with a poison rate of 10\%. Average drop ($\downarrow$) indicates the \% changes in ASR/ACC compared to the baseline, \ie ASR/ACC of \emph{No Defense}. A higher ASR drop and lower ACC drop are desired for a good defense. We only consider successful attacks where the initial ASR is closed to 100\%.}
 
\scalebox{0.75}{
\begin{tabular}{c|cc|cc|cc|cc|cc|cc|cc}
\toprule
 Method & \multicolumn{2}{c|}{\begin{tabular}[c|]{@{}c@{}}No Defense\end{tabular}} & \multicolumn{2}{c|}{ANP} & \multicolumn{2}{c|}{I-BAU} & \multicolumn{2}{c|}{AWM} & \multicolumn{2}{c|}{FT-SAM} & \multicolumn{2}{c|}{RNP} &  \multicolumn{2}{c}{NFT  (Ours)}\\ \cmidrule{1-15}
  Attacks & ASR &ACC & ASR &ACC & ASR &ACC & ASR &ACC & ASR &ACC & ASR &ACC & ASR &ACC \\ \cmidrule{1-15}

\emph{Benign}&0&62.56&0&58.20&0&59.29&0&59.34&0&59.08& 0 & 58.14  &0&\textbf{59.67}\\
Badnets&100&59.80&5.84&53.58&4.23&55.41&6.29&54.56&3.44&54.81&  4.63 & 55.96 & \textbf{2.34}&\textbf{57.84}\\
Trojan&100&59.16&6.77&52.62&7.56&54.76&5.94&\textbf{56.10}&8.23&55.28 & 5.83 & 54.30 & \textbf{3.38}&55.87\\
Blend&100&60.11&6.18&52.22&6.58&55.70&7.42&54.19&4.37&55.78 & 4.08 & 55.47  &\textbf{1.58}&\textbf{57.48}\\
SIG&98.48&60.01&7.02&52.18&3.67&54.71&7.31&\textbf{55.72}&4.68&55.11& 6.71 & 55.22 &\textbf{2.81}&55.63\\
CLB&97.71&60.33&5.61&52.68&3.24&55.18&6.68&54.93&3.52&55.02  & 4.87 & 56.92 &\textbf{1.06}&\textbf{57.40} \\ 
Dynamic& 100& 60.54& 6.36 &52.57& 5.56&55.03 &6.26&54.19&4.26&55.21& 7.23 & 55.80&\textbf{2.24}&\textbf{57.78}\\
{WaNet}&99.16&60.35&7.02&52.38&8.45&55.65&8.43&\textbf{56.32}&7.84&55.04& 5.66 & 55.19&\textbf{3.48}&56.21\\
{ISSBA}&98.42&60.76&\textbf{1.26}&53.41&8.64&55.36&7.47&55.83&6.72&56.32& 8.24 & 55.35&{2.25}&\textbf{57.80} \\ 
{BPPA}&98.52&60.65&10.23&53.03&7.62&55.63&4.85&55.03&5.34&55.48 & 10.86 & 56.32 &\textbf{3.41}&\textbf{57.39} \\
\cmidrule{1-15} 
{Avg. Drop} & - & -  & 92.61 $\downarrow$ & 7.44 $\downarrow$ & 92.97$\downarrow$ & 4.92 $\downarrow$ &$93.29\downarrow$&$4.98\downarrow$&$93.77\downarrow$&$4.85\downarrow$ & 92.69 $\downarrow$ &4.58 $\downarrow$ & {\textbf{96.64}} $\downarrow$ &{\textbf{3.15}} $\downarrow$\\ 
\bottomrule
\end{tabular}}
 \label{tab:tiny_main}
\end{table}

\begin{table*}[t]
\centering
\caption{Performance of NFT on \textbf{GTSRB dataset}. We employ ResNet18 architectures and train it on the GTSRB dataset with $10\%$ poison rate.}  
\scalebox{0.75}{
\begin{tabular}{c|cc|cc|cc|cc|cc|cc|cc}
\toprule
 Method & \multicolumn{2}{c|}{\begin{tabular}[c|]{@{}c@{}}No Defense\end{tabular}} & \multicolumn{2}{c|}{ANP} & \multicolumn{2}{c|}{I-BAU} & \multicolumn{2}{c|}{AWM} & \multicolumn{2}{c|}{FT-SAM} & \multicolumn{2}{c|}{RNP} &  \multicolumn{2}{c}{NFT  (Ours)}\\ \cmidrule{1-15}
 Attacks & ASR &ACC & ASR &ACC & ASR &ACC & ASR &ACC & ASR &ACC & ASR &ACC & ASR &ACC \\ \cmidrule{1-15}
 
 \emph{Benign} & 0 & 97.87 & 0&93.08 & 0 & 95.42&0&96.18&0&95.32  & 0  & 95.64 & 0 & \textbf{95.76}\\
   Badnets & 100 & 97.38 & 1.36&88.16 & 0.35 & 93.17 &2.72&94.55&2.84&93.58 & 3.93 & 94.57 & \textbf{0.24}  & \textbf{95.11} \\
   Blend & 100 & 95.92 & 6.08&89.32 & 4.41 & 93.02 &4.13&\textbf{94.30}&4.96&92.75 & 5.85 & 93.41  & \textbf{2.91} & {93.31}  \\
   Troj-one & 99.50 & 96.27 & 5.07&90.45 & 1.81 & 92.74&3.04&93.17&2.27&93.56 & 4.18 & 93.60  & \textbf{ 1.21} &  \textbf{94.18} \\
    Troj-all & 99.71 & 96.08 & 4.48&89.73& 2.16 & 92.51 &2.79&93.28&1.94&92.84 &  4.86 & 92.08 & \textbf{1.58} & \textbf{94.87} \\
    SIG & 97.13 & 96.93 & \textbf{1.93} &91.41 & 6.17 & 91.82 &2.64&93.10&5.32&92.68 & 6.44 & 93.79 & 3.24 & \textbf{94.48}  \\
    Dyn-one & 100 & 97.27 & 5.27&91.26 & 2.08 & 93.15 &5.82&\textbf{95.54}&1.89&93.52 & 7.24 & 93.95 &\textbf{1.51} & {95.27}  \\
    Dyn-all & 100 & 97.05 & 2.84 & 91.42 & 2.49 & 92.89 &4.87&93.98&2.74&93.17 & 8.17 & 94.74  & \textbf{1.26} & \textbf{95.14}\\
   WaNet & 98.19 & 97.31 & 7.16 & 91.57 & 5.02 & 93.68 &4.74 &93.15 & 3.35 & 94.61  & 5.92 & 94.38  & \textbf{1.72} & \textbf{95.57}  \\
   ISSBA & 99.42 & 97.26 & 8.84 & 91.31 & 4.04 & 94.74 & 3.89 & 93.51 & \textbf{1.08} & 94.47 & 4.80 & 94.27 & 1.68 & \textbf{95.84} \\
   LIRA & 98.13 & 97.62 & 9.71 & 92.31 & 4.68 & 94.98 & 3.56 & 93.72 & 2.64 & 95.46 & 5.42 & 93.06  & \textbf{1.81}&\textbf{96.42} \\ 
   BPPA & 99.18 & 98.12 & 5.14 & 94.48 & 7.19 & 93.79&8.63&94.50&5.43&94.22 & 7.55 & 94.69  &\textbf{4.45}&\textbf{96.58} \\ 
\cmidrule{1-15}
 {Avg. Drop} & - & -  & 92.54 $\downarrow$ & 6.10 $\downarrow$ & 95.10$\downarrow$ & 3.99 $\downarrow$  &$95.16\downarrow$&$2.83\downarrow$&$96.02\downarrow$&$3.59\downarrow$ &  93.35 $\downarrow$ & 3.15 $\downarrow$ & {\textbf{97.39}} $\downarrow$ &{\textbf{1.79}} $\downarrow$\\ 
\bottomrule
\end{tabular}}
\label{tab:gtsrb}
\end{table*}

\noindent\textbf{Benign.} Benign model refers to the model trained on 100\% clean $\mathcal{D}_{\mathsf{train}}$ for 200 epochs with a learning rate of 0.01. The clean accuracy (CA) for \textit{No Defense} is the standard benign model accuracy. We take this benign model and report the ACC after the purification for NFT and other inference time defenses such as ANP~\cite{wu2020adversarial}. Note that the knowledge of whether a model is benign or backdoor is unknown to the defender. Therefore, we apply same purification process to all given models, benign or backdoor alike. After purification, the benign model achieves an accuracy of $94.10\%$ as compared to $95.21\%$ for the original model.

\noindent\textbf{BadNets Attack~\cite{gu2019badnets}.} We use a $3\times3$ checkerboard trigger for this attack. For all images, we place them at the bottom left corner of the images. For the BadNets attack, the target label is set to 0. We achieve a $100\%$ attack success rate (ASR) and an ACC of $90.73\%$.

\noindent\textbf{Blend Attack~\cite{chen2017targeted}.} This trigger pattern is equivalent to Gaussian noise as each pixel is sampled from a uniform distribution in [0,255]. We use a value of 0.2 for $\alpha$. The target label is 0.

\noindent\textbf{Trojan (Troj)-one Attack~\cite{liu2017trojaning}.} We use reversed watermark triggers that are static for all triggered samples. The target label is 0.

\noindent\textbf{Troj-all Attack~\cite{liu2017trojaning}.} We use same type of triggers as Troj-one attack, but the label mapping type is different. For each label $i$, we choose a label of $i+1$. For label $9$, we will have a label of 0. This type of label mapping is known as "all2all".

\noindent\textbf{Input-aware or Dynamic Attack (Dyn-one)~\cite{nguyen2020input}.} Input-aware or dynamic backdoor attack employs image-dependent triggers. Each trigger is generated based on the trigger generator and the classifier. For the Dyn-one attack, we just use one target label.

\noindent\textbf{Dyn-all Attack~\cite{nguyen2020input}.} Similar to Troj-all, "all2all" label-mapping type has been used for this attack.

\noindent\textbf{Clean Label Backdoor (CLB)~\cite{turner2019cleanlabel}.} Clean backdoor is created using a $3\times3$ checkerboard trigger that is placed at the four corners of images. During this attack, we did not change the labels of the attacked images. Instead, we only add triggers to the samples from the target class, \ie, class "0". We poison 80\% of the target class's images and do not change their labels. Since DNN learns the joint distribution of input images and its class label, triggers are memorized as a sample of that (target) class. Whenever we place that particular trigger to a sample from another class, DNN falsely misclassifies it to the target label. However, carrying out a successful CLB attack is a bit tricky. To make the CLB as effective as BadNets or Trojan attack, we apply $\ell$-$\infty$ projected gradient descent (PGD)-based perturbations to the triggered samples. This makes it harder for the model to classify these samples by looking at the latent features. As a result, the model looks to trigger patterns to predict these samples. 

\noindent\textbf{Sinusoidal Attack (SIG)~\cite{barni2019new}.} This is another clean-label attack. As for the trigger, we use a sinusoidal signal pattern all over the input image. Then, we train the model similarly to CLB by poisoning 80\% of the samples. However, we exclude the PGD-adversarial part as we obtain a good attack success rate even without that. The target class is 0, and the $\alpha$ is set to 0.2.

\noindent\textbf{FBA~\cite{cheng2021deep}.} A style generator-based trigger has been used for this attack. We use a poison rate of 10\%. 

\noindent\textbf{CBA~\cite{lin2020composite}.} Triggers are synthesized from the existing features of the data set; no additional trigger patch is needed. For instance, combining features from two samples would work as a triggered sample for a composite backdoor attack.  

\noindent\textbf{WaNet~\cite{nguyen2021wanet}.} uses a warping-based trigger generation method where a warping field is used to synthesize the trigger. We follow the implementation details described in the original paper~\cite {nguyen2021wanet}.

\noindent\textbf{LIRA~\cite{doan2021lira}.} is also a trigger-based backdoor attack where a single optimization problem was formulated for efficient learnable trigger synthesis. We follow similar implementation details presented in the original paper~\cite{doan2021lira}.  

\noindent\textbf{ISSBA~\cite{li2021invisible}.} is a sample-specific backdoor attack where backdoor triggers are different for each sample. Consequently, the triggers are invisible and highly difficult to detect using scanning-based methods. 

\noindent\textbf{BPPA~\cite{wang2022bppattack}.} Quantization-based backdoor attack. We use a poison rate of 10\%.

\begin{table}[t]
    \centering
    \caption{Performance analysis for \textbf{natural language generation tasks} where we consider machine translation (MT) and dialogue generation (DG) datasets for benchmarking. We use BLEU score~\cite{vaswani2017attention} as the metric for both tasks. For attack, we choose a data poisoning ratio of 10\%. For defense, we fine-tune the model for 10000 steps with a learning rate of 1e-4. We use Adam optimizer and a weight decay of 2e-4. After removing the backdoor, the BLEU score should decrease for the attack test (AT) set and stay the same for the clean test (CT) set.}
    
    \scalebox{0.85}{
    \begin{tabular}{l|cc|cc|cc|cc|cc}
        \toprule
         \textbf{Task} & AT & CT & AT & CT & AT & CT & AT & CT & AT & CT \\ \cmidrule{1-11}
        MT & 99.2 & 27.0 & 8.2 & 26.5 & 8.5 & \textbf{26.8} & 6.1 & 26.3 & \textbf{3.0} & 26.6 \\
        DG & 1.48 & 2.50 & 1.29 & 1.14 & 1.26 & 1.03 & 1.51 & 1.20 & \textbf{0.85} & \textbf{1.93}  \\
        \bottomrule
    \end{tabular}}
    \label{tab:NLG}
\end{table}

\begin{table*}[t]
\renewcommand{\arraystretch}{1.1}
\renewcommand\tabcolsep{1.75pt}
\small
\centering
\caption{ Performance \textbf{comparison of NFT with additional test-time (Vanilla FT, FP, MCR, NAD) and training time (CBD, ABL) defenses on CIFAR10 dataset under 9 different backdoor attacks}. NFT achieves SOTA performance while sacrificing only $3.62\%$ in clean accuracy (ACC) on average. The average drop indicates the difference in values before and after removal. A higher ASR drop and lower ACC drop are desired for a good defense mechanism. Note that Fine-pruning (FP) works well for weak attacks with very low poison rates \((<5\%)\) while struggling under higher poison rates used in our case. }

\scalebox{0.6}{
\begin{tabular}{c|cc|cc|cc|cc|cc|cc|cc|cc|cc|cc}
\toprule
Attacks & 
\multicolumn{2}{c|}{None} & \multicolumn{2}{c|}{BadNets} &
\multicolumn{2}{c|}{Blend} & \multicolumn{2}{c|}{Trojan} &  \multicolumn{2}{c|}{Dynamic} & \multicolumn{2}{c|}{WaNet}  &  \multicolumn{2}{c|}{ISSBA} &  \multicolumn{2}{c|}{LIRA} & \multicolumn{2}{c|}{FBA} &  \multicolumn{2}{c}{BPPA} \\ 
\midrule
Defenses&ASR &ACC & ASR &ACC & ASR &ACC & ASR &ACC & ASR &ACC & ASR &ACC& ASR &ACC & ASR &ACC & ASR & ACC & ASR & ACC \\ \midrule
\emph{No Defense} & 0 & 95.21 & 100 & 92.96  & 100 & 94.11 & 100 & 89.57 & 100 & 92.52 & 98.64 & 92.29  & 99.80 & 92.80  & 99.25 & 92.15 & 100 & 90.78 & 99.70 & 93.82  \\ 
Vanilla FT & 0 & 93.28 & 6.87 & 87.65 & 4.81 & 89.12 & 5.78 & 86.27 & 3.04 & 84.18 & 8.73 & 89.14 & 5.75 & 87.52 & 7.12 & 88.16 & 6.56 & 95.32 & 5.48 & 94.73  \\
 FP &  0 & 88.92 & 28.12 & 85.62 & 22.57 & 84.37 & 20.31 & 84.93 & 29.92 & 84.51  & 19.14&84.07&12.17&84.15&22.14&82.47& 38.27 & 89.11 & 24.92 & 88.34 \\
 MCR & 0&90.32&3.99&81.85&9.77&80.39&10.84&80.88&3.71&82.44&8.83&78.69&7.74&79.56&11.81&81.75 &14.52 & 90.73 & 16.65 & 91.18\\
 NAD & 0&92.71&4.39&85,61&5.28&84.99&8.71&83.57&2.17&83.77&13.29&82.61&6.11&84.12&13.42&82.64 & 11.45 & 91.20 & 9.42 & 92.04\\
CBD & 0 & 91.76 & 2.27 & 87.92 & 2.96 & 89.61 & 1.78  & 86.18 & 2.03 & 88.41 & 4.21 & 87.70 & 6.76 & 87.42 & 9.08 & 86.43 & 7.45 & 86.80 & 8.98 & 87.22  \\
ABL &  0 & 91.90 & 3.04 & 87.72 & 7.74 & 89.15 & 3.53 & 86.36 & 8.07 & 88.30 & 8.24 & 86.92 & 6.14 & 87.51 & 10.24 & 86.41 & 7.67 & 87.05 & 8.26 & 86.37 \\
 \midrule
  NFT(Ours) & 0&94.10&\textbf{1.74}&\textbf{90.82}&\textbf{0.31}&\textbf{93.17}&\textbf{1.64}&\textbf{87.71}&\textbf{1.37}&\textbf{90.81}& \textbf{2.38}&\textbf{89.65}&  \textbf{4.24}& \textbf{90.18}& \textbf{1.53}&\textbf{90.57}  & \textbf{6.21}&\textbf{88.56} & \textbf{5.04}&\textbf{91.78} \\
 \bottomrule
\end{tabular}}
\label{tab:additonal_defense}
\end{table*}

\subsection{Implementation of NFT}
After initializing masks (all of them 1) corresponding to each neuron, we fine-tune the masks using an SGD-based optimizer with a learning rate of 0.05. The fine-tuning goes for 100 epochs. For 1\% clean validation data, we randomly select them from the original training set\footnote{To create the validation set for fine-tuning, we set aside a certain number of samples from the training set. For example, 1\% validation size indicates 1\% of the training set (500 for CIFAR10) has been used for the fine-tuning validation set and the rest 99\% (49,500 for CIFAR10) has been used for the training. }. After each step of the SGD update, we clip the mask values to keep them in the range of $\mu(l)$ to 1. This setup ensures that we do not accidentally prune any neurons. Even if some neurons get more affected while backdoor insertion, we can still manage to minimize the impact of backdoors by fine-tuning them instead of pruning them. Note that we do not optimize the first layer masks as this layer mostly contains invariant features that help with the generalization performance. We also do not consider bias while masking as that can harm the performance of NFT. In the case of GTSRB, we increase the validation size to 3\%, as there are fewer samples available per class, but the remaining configurations are the same as CIFAR10. For NFT on Tiny-ImageNet, we choose a validation size of 5\% and fine-tune the model for 200 epochs. Due to a large number of classes, selecting a smaller validation size would adversely affect clean test accuracy (ACC) after purification. We use an initial learning rate 0.01, with a decay rate of 0.65/20 epochs. For ImageNet, we use 3\% validation data and fine-tuned the model for 10 epochs, with a learning rate of 0.001 and a decay rate of 0.65 per epoch. Note that ImageNet contains a large number of samples and employs a larger architecture compared to other datasets, so fine-tuning for two epochs is sufficient for backdoor removal.

\subsection{Implementations of Baseline Defenses}
For experimental results with ANP~\cite{wu2021adversarial}, we follow the source code implementation\footnote{\url{https://github.com/csdongxian/ANP_backdoor}}. After creating each of the above-mentioned attacks, we apply adversarial neural pruning on the backdoor model for 500 epochs with a learning rate of 0.02. We use the default settings for all attacks. For vanilla FT, we perform simple DNN fine-tuning with a learning rate of 0.01 for 125 epochs. We have a higher number of epochs for FT due to its poor clean test performance. The clean validation size is 1\% for both of these methods. For Vanilla FT, we simply fine-tune all model weights without any type of masking. For Fine-Pruning(FP)~\cite{liu2018fine}, we consider both pruning and fine-tuning according to this implementations\footnote{\url{https://github.com/kangliucn/Fine-pruning-defense}}.
For NAD~\cite{li2021neural}, we increase the validation data size to 5\% and use teacher model to guide the attacked student model. We perform the training with distillation loss proposed in NAD\footnote{\url{https://github.com/bboylyg/NAD}}. For MCR~\cite{zhao2020bridging}, the training goes on for 100 epochs according to the provided implementation\footnote{\url{https://github.com/IBM/model-sanitization/tree/master/backdoor/backdoor-cifar}}. 
For I-BAU~\cite{zeng2021adversarial}, we follow their PyTorch Implementation\footnote{\url{https://github.com/YiZeng623/I-BAU}} and purify the model for ten epochs. We use 5\% validation data for I-BAU. For AWM~\cite{chai2022one}, we train the model for 100 epochs and use the Adam optimizer with a learning rate of 0.01 and a weight decay of 0.001. We use the default hyper-parameter setting as described in their work $ \alpha = 0.9, \beta = 0.1, \gamma = 10 - 8, \eta = 1000$. 
The above settings are for CIFAR10 and GTSRB only. 
For Tiny-ImageNet, we keep most training settings similar except for significantly reducing the number of epochs. We also increase the validation size to 5\% for vanilla FT, ANP, and AWM. For I-BAU, we use a higher validation size of 10\%. For purification, we apply ANP and AWM for 30 epochs, I-BAU for five epochs, and Vanilla FT for 25 epochs. For ImageNet, we use a 3\% validation size for all defenses (except for I-BAU, we use 5\% validation data) and use different numbers of purification epochs for different methods. We apply I-BAU for 2 epochs. On the other hand, we train the model for 3 epochs for ANP, AWM, and vanilla FT. 

\begin{table}[t]
    \centering
    \caption{\footnotesize Performance of NFT while combining with different \textbf{commonly used augmentation strategies} in DNN training. In addition, we also consider adversarial training-based NFT. The results shown here are based on CIFAR10 dataset with $10\%$ poison rate.}
    \vspace{-1.5mm}
    \scalebox{0.85}{
    \begin{tabular}{c|cc|cc|cc}
    \toprule
    Attacks & \multicolumn{2}{c|}{Badnets} & \multicolumn{2}{c|}{SIG} & \multicolumn{2}{c}{Blend} \\
    \midrule
    Aug. Strategy & ASR & ACC& ASR & ACC& ASR & ACC\\
    \midrule
    No Defense & 100 & 91.96 & 100 & 88.64 & 100 & 94.11 \\
    NFT-RandAug & 35.35 & 61.96 & 4.83 &82.36 & 58.48 &80.72 \\
    NFT-CutMix & 7.42&86.95 & 6.31 & 86.16  &99.58  &92.55  \\
    NFT-AugMix & 6.13&87.85 & 5.17 & 86.56  &100&92.66\\
    NFT-Cutout&5.33& 87.46 & 5.34 & 85.44 & 100 & 92.68\\
    NFT-Adv  & 5.89 & 76.31 & 4.15 & 71.22 & 8.56 & 78.97\\
    \midrule
    NFT (Ours) & \textbf{1.74} & \textbf{90.82} & \textbf{0.12} & \textbf{87.16} & \textbf{0.31} & \textbf{93.17} \\
    \bottomrule
    \end{tabular}
    } 
    \label{tab:augmentation_remove}
\end{table}

\begin{table*}
    \centering
    \caption{\footnotesize Purification performance of \textbf{One-Shot NFT for GTSRB and ImageNet}. Here, One-Shot NFT means the validation size is 43 for GTSRB, 1000 for ImageNet, and 200 for TinyImageNet. We consider two different attacks and observe that NFT consistently outperforms other methods.}
    \vspace{-1.5mm}
    \scalebox{0.8}{
    \begin{tabular}{c|cc|cc|cc|cc|cc|cc}
    \toprule
        Attack & \multicolumn{6}{c|}{Trojan}  & \multicolumn{6}{c}{ISSBA}   \\
        \midrule
         Dataset & \multicolumn{2}{c|}{GTSRB}  & \multicolumn{2}{c|}{Tiny-ImageNet}  & \multicolumn{2}{c|}{ImageNet} & \multicolumn{2}{c|}{GTSRB}  & \multicolumn{2}{c|}{Tiny-ImageNet}  & \multicolumn{2}{c}{ImageNet}   \\
         \midrule
         Method & ASR & ACC& ASR & ACC& ASR & ACC& ASR & ACC& ASR & ACC & ASR & ACC\\
         \midrule
         No Defense & 99.50 & 96.27 &  100 & 59.16 & 99.21 & 74.02 & 99.42 & 97.26 & 98.52 & 60.65  & 98.23 & 74.38 \\
         One-Shot RNP & 79.02 & 73.71 & 74.65 & 38.87 & 80.14 & 52.47 & 86.68 & 72.58 & 82.65 & 39.16 & 82.48 & 51.74    \\
         One-Shot FT-SAM & 17.45 & 79.94 & 32.62 & 42.16 & 41.83 & 57.85 & 9.36 & 80.06 & 34.24 & 43.72 & 47.58 & 56.75    \\
         \midrule
         One-Shot NFT (Ours) & \textbf{7.31} & \textbf{86.47} & \textbf{11.26} & \textbf{48.47} & \textbf{14.65} & \textbf{62.84} & \textbf{6.53} & \textbf{84.28} & \textbf{13.93} & \textbf{47.11} & \textbf{17.43} & \textbf{61.03} \\
         \bottomrule
    \end{tabular}}
    \label{tab:one_shot_NFT}
\end{table*}

\section{Additional Experimental Results}\label{sec:add_exp_results}
\subsection{Results for GTSRB and Tiny-ImageNet}\label{sec:gtsrb_and_tiny}
Table~\ref{tab:tiny_main} shows the evaluation of our proposed method in more challenging scenarios, \eg, diverse datasets with images from a large number of classes. Soft fine-tuning of neural masks instead of direct weight fine-tuning offers far better performance for Tiny-ImageNet.  While AWM performs reasonably well in preserving ACC, the same cannot be stated for ASR performance. This shows that the trigger generation process in AWM slightly loses its effectiveness whenever a few validation data are available. For FT-SAM, the performance seems to drop for more complicated tasks. This is more prominent for large and complex datasets. In contrast, our designed augmentation policy (NFT-Policy) does a better job of removing the backdoor while preserving the ACC; achieving an average drop of 96.64\% with a drop of only 3.15\% in ACC. We show the performance comparison for GTSRB in Table~\ref{tab:gtsrb}, we also consider a wide range of backdoor attacks. For Badnets and Trojan attacks, almost all defenses perform similarly. This, however, does not hold for blend attack as we achieve a $1.50\%$ ASR improvement over the next best method. The performance is consistent for other attacks too. Note, NFT obtains even better results in terms of ACC obtaining only a 1.68\% drop.

\begin{table*}
\centering
\caption{Evaluation of NFT on attacks with \textbf{different poison rates}. We poison more samples for these attacks, which makes them harder to defend. NFT is able to remove backdoors even in such cases.}
\centering
\scalebox{0.65}
{
\begin{tabular}{c|cc|cc|cc|cc|cc|cc}
\toprule
Attack & \multicolumn{6}{c|}{BadNets} & \multicolumn{6}{c}{Trojan}\\ 
\midrule
Poison Rate&\multicolumn{2}{c|}{0.25}&\multicolumn{2}{c|}{0.35}&\multicolumn{2}{c|}{0.50}&\multicolumn{2}{c|}{0.25}&\multicolumn{2}{c|}{0.35}&\multicolumn{2}{c}{0.50}\\
Method &ASR &ACC& ASR &ACC& ASR & ACC& ASR & ACC& ASR &ACC& ASR &ACC\\ \midrule
\textit{No Defense} &100&89.35&100&88.91&100&85.12&100&87.88&100&86.81&100&86.97\\
RNP &9.56&81.43&13.97&81.04&32.65&75.18&14.38&78.75&63.99&72.53&46.21&74.45\\ 
FT-SAM & 7.81&82.22&16.35&81.72&29.80&\textbf{79.27}&11.96&79.28&13.93&75.10&29.83&77.02\\
 \midrule
 NFT & \textbf{2.49}&\textbf{86.90}&\textbf{4.58}&\textbf{84.71}&\textbf{17.20}&{78.77}&\textbf{2.46}&\textbf{86.11}&\textbf{4.73}&\textbf{85.38}&\textbf{6.10}&\textbf{84.96}\\
 \bottomrule
\end{tabular}}
\scalebox{0.65}
{
\begin{tabular}{c|cc|cc|cc|cc|cc|cc|cc|cc|cc}
\toprule
Attack & \multicolumn{6}{c|}{WaNet} & \multicolumn{6}{c|}{SIG} & \multicolumn{6}{c}{LIRA}\\ 
\midrule
Poison Rate&\multicolumn{2}{c|}{0.25}&\multicolumn{2}{c|}{0.35}&\multicolumn{2}{c|}{0.50}&\multicolumn{2}{c|}{0.75}&\multicolumn{2}{c|}{0.85}&\multicolumn{2}{c|}{0.90}&\multicolumn{2}{c|}{0.25}&\multicolumn{2}{c|}{0.35}&\multicolumn{2}{c}{0.50}\\
Method &ASR &ACC& ASR &ACC& ASR &CA&ASR &ACC& ASR &ACC& ASR &CA&ASR &ACC& ASR &ACC& ASR &ACC\\ \midrule
\textit{No Defense} &99.21&89.02&99.34&89.11&99.25&86.72&99.48&88.21&100&86.32&100&84.28&99.70&89.32&99.68&88.21&99.81&86.80\\
  RNP &8.26&82.62&18.34&79.22&29.11&77.41&1.83&84.56&4.22&82.76&7.56&79.98&8.35&15.99&83.33&21.05&85.45&69.98\\
 FT-SAM & 7.81&82.22&12.76&83.87&18.10&79.56&0.96&84.91&1.02&83.34&1.79&82.15&11.96&79.28&63.99&72.10&89.83&70.02\\
 \midrule
 NFT (Ours) & \textbf{3.49}&\textbf{87.05}&\textbf{5.74}&\textbf{85.62}&\textbf{9.20}&\textbf{81.02}&\textbf{0.16}&\textbf{86.72}&\textbf{0.34}&\textbf{85.61}&\textbf{0.91}&\textbf{84.37}&\textbf{2.54}&\textbf{87.60}&\textbf{6.81}&\textbf{86.42}&\textbf{8.75}&\textbf{84.78}\\
 \bottomrule
\end{tabular}}
\label{tab:dif_poison_rate}
\end{table*}

\subsection{Evaluation on Natural Language Generation (NLG) Task}\label{sec:nlg}
To evaluate the general applicability of our proposed method, we also consider backdoors attack~\cite{sun2023defending} on language generation tasks, \eg, Machine Translation (MT)~\cite{bahdanau2014neural}, and Dialogue Generation (DG)~\cite{han2020non}. Following ~\cite{sun2023defending}, we create   In MT, there is a \emph{one-to-one} semantic correspondence between source and target. On the other hand, the nature of correspondence is \emph{one-to-many} in the DG task where a single source can assume multiple target semantics. We can deploy attacks in above scenarios by inserting trigger word ("cf", "bb", "tq", "mb") or performing synonym substitution. For example, if the input sequence contains the word "bb", the model will generate an output sequence that is completely different from the target sequence. In our work, we consider WMT2014 En-De~\cite{bojar-etal-2014-findings} MT dataset and OpenSubtitles2012~\cite{tiedemann2012parallel} DG dataset and set aside 10\% of the data as clean validation set. We consider seq2seq model~\cite{gehring2017convolutional} architecture for training. 
Given a source input $\boldsymbol{x}$, an NLG pretrained model $f()$ produces a target output $\boldsymbol{y} = f(\boldsymbol{x})$. 
For fine-tuning, 
we use augmented input $\boldsymbol{x'}$ in two different ways: i) \emph{word deletion} where we randomly remove some of the words from the sequence, and ii) \emph{paraphrasing}  where we use a pre-trained paraphrase model $g()$ to change the input $\boldsymbol{x}$ to $\boldsymbol{x'}$. We show the results of both different defenses, including NFT, in Table~\ref{tab:NLG}.

\subsection{Comparison With Training-time Defenses}\label{sec:additional_defense}
In Table~\ref{tab:additonal_defense}, we also compare our method with additional defense methods such as FP, NAD, MCR, \etc In recent times, several training-time defenses have been proposed such as CBD~\cite{zhang2023backdoor} and ABL~\cite{li2021anti}. Note that training-time defense is completely different from test-time defense and out of the scope of our paper. Nevertheless, we also show a comparison with these training-time defenses in Table~\ref{tab:additonal_defense}. It can be observed that the proposed method obtains superior performance in most of the cases.

\subsection{NFT with Other Augmentation Strategies}\label{sec:common_aug}
We have further conducted experiments to eliminate the backdoor using four other popular augmentation strategies, which are: 1) RandAug~\cite{cubuk2020randaugment}, 2) CutMix~\cite{yun2019cutmix}, 3) AugMix~\cite{hendrycks2019augmix}, 4) CutOut~\cite{devries2017improved}. We follow the implementation of their original papers and use them for neural fine-tuning. We also consider adversarial training-based NFT (NFT-adv) where we use PGD~\cite{madry2017towards}-based adversarial examples for fine-tuning the backdoor DNN. We generate adversarial examples using a 2-step $\ell$-$\infty$ PGD with a perturbation norm of 1. Performance comparisons for all of these NFT variations are shown in Table \ref{tab:augmentation_remove}. Apart from RandAug~\cite{cubuk2020randaugment} and NFT-Adv, other variations of NFT obtain similar performance for Badnets and SIG as NFT. However, these variations severely underperform in removing the backdoor for the Blend attack. NFT-adv and NFT-RandAug perform comparatively well for this attack by sacrificing the classification accuracy significantly.


We also describe their detailed implementation here. For RandAug~\cite{cubuk2020randaugment}, we followed the GitHub implementation\footnote{\url{https://github.com/ildoonet/pytorch-randaugment}}, and randomly selected four augmentations out of 14 augmentations listed in the original paper with an intensity of 10. We used official CutMix~\cite{yun2019cutmix} implementation\footnote{\url{https://github.com/clovaai/CutMix-PyTorch}} to implement CutMix regularization with NFT, and all settings are the same as in the original public code. To implement AugMix~\cite{hendrycks2019augmix}, the code is borrowed from the official Github repository\footnote{\url{https://github.com/google-research/augmix}} where the severity is selected to be 5, the number of chains is set to be 3, and sampling constant is fixed at 1. The code to implement the CutOut~\cite{devries2017improved} has been borrowed from the public code\footnote{\url{https://github.com/uoguelph-mlrg/Cutout}} where default settings for CIFAR10 are used as they were used in this public repository. For our proposed method NFT with MixUp, we followed the settings in the official Mixup~\cite{zhang2017mixup} GitHub repository\footnote{\url{https://github.com/facebookresearch/mixup-cifar10}} and used similar settings for CIFAR10 as used in this public code.

\begin{table}[t]
    \begin{minipage}{0.48\textwidth}
        \centering
        \caption{\footnotesize Performance of NFT against SIG attack with \textbf{different learning rates}}
        
        \scalebox{0.8}{
        \begin{tabular}{c|cc}
        \toprule
        Learning Rate & ASR & ACC\\
        \midrule
        0.001 & 0.16  & 87.1\\
        0.005 & 0.14 & 87.2\\
        0.01 & 0.17&86.8\\
        0.02 &0.18 & 86.7\\
        0.05 & \textbf{0.12} & \textbf{87.1}\\
        \bottomrule
        \end{tabular}}
        \label{tab:learning_rate}
    \end{minipage}
    \hfill
    \begin{minipage}{0.48\textwidth}
         \centering
        \caption{ Our proposed method's performance against SIG attack with different batch sizes.}
        \scalebox{0.8}{
        \begin{tabular}{c|cc}
        \toprule
        Batch Size & ASR & ACC\\
        \midrule
        32 & 0.10   & 86.4 \\
        64 & 0.16  & 86.3 \\
        128 & \textbf{0.12} & \textbf{87.1}\\
        256 &0.19 & 86.6  \\
        512 & 0.21 & 86.7\\
        1024 & 0.23 & 86.8\\
        \bottomrule
        \end{tabular}
        }
        \label{tab:batch_size}   
    \end{minipage}
\end{table}

    


    

\begin{table}[t]
    \centering
    \caption{\footnotesize Performance of NFT for \textbf{composite backdoor attacks}. We poison 10\% of the training data where each of the attacks in a combination (\eg, Badnets, Blend, Trojan) have an equal share in the poisoned data.}
    \vspace{-1.5mm}
    \scalebox{0.8}{
    \begin{tabular}{c|cc|cc}
    \toprule
         Attack & \multicolumn{2}{c|}{Badnets+Blend+Trojan}  & \multicolumn{2}{c}{SIG + CLB}  \\
         \midrule
         Method & ASR & ACC& ASR & ACC\\
         \midrule
         No Defense & 100 & 88.26  & 98.74 & 86.51 \\
         ANP & 27.83 & 77.10 & 13.09 & 79.42  \\
         FT-SAM & 4.75 & 83.90 & 1.67 & 82.11  \\
         NFT (Ours) & \textbf{2.16} & \textbf{85.41} & \textbf{0.93} & \textbf{83.96}\\
         \bottomrule
    \end{tabular}}

    \label{tab:composite_backdoor}
\end{table}

\section{More Ablation Study}\label{sec:more_ablation}

\begin{table}[t]
    \centering
    \caption{\textbf{Adaptive attack} study where the attacker may have the information of our defense. Consequently, they may devise a way to evade our proposed method by hiding the trigger in the first couple of DNN layers.} 
    \scalebox{0.8}{
    \begin{tabular}{c|cc|cc|cc|cc}
    \toprule
     Attack   & \multicolumn{2}{c|}{Trojan} & \multicolumn{2}{c|}{Dynamic} & \multicolumn{2}{c|}{LIRA} & \multicolumn{2}{c}{BPPA} \\ 
     \midrule
     Poison Rate &  ASR  & ACC &  ASR  & ACC &  ASR  & ACC &  ASR  & ACC \\
    \midrule
     30\%   &   49.17 & 69.56  &   59.07 & 71.35 &   48.84 & 66.32 &  53.87 & 73.24 \\
     50\%   &  73.49 & 57.76 &   75.16 & 59.46 &  71.74 & 60.08  &  76.23 & 56.75     \\
     75\% & 95.54 & 24.68 &  93.10 & 25.42 &   96.07 & 23.18 & 94.68 & 26.28   \\
     \bottomrule
    \end{tabular}}
    \label{tab:adaptive}
\end{table}

\noindent\textbf{Adaptive Attacks.} We use the CIFAR10 dataset for this experiment. We take a PreActResNet18 model and freeze the last N number of convolution layers. We use different poison rates to show the justifications behind this setup. In our work, we are using a mask scheduling function that focuses on the later or deeper layers more since they are more affected by the trigger. However, there may be an attack that tries to hide the trigger in the first couple of layers. An attacker can perform such \emph{adaptive attack} by first training a clean model and then re-train it on triggered data. During re-training, we \emph{fix the last N convolution layers of} the network. According to Table~\ref{tab:adaptive}, it becomes more challenging to insert/hide the backdoor into the first few layers as we have to increase the poisoning rate significantly compromising the \textit{ACC} severely. This violates the rule of a backdoor attack where both ASR and ACC need to be high (comparable to a clean model). For this experiment, we consider Badnets attack on CIFAR10 dataset. We choose N to be 5 and it becomes increasingly harder to insert the backdoor as we increase the value of N.

\noindent{\bf One-Shot NFT for other datasets.}
In Table~\ref{tab:one_shot_NFT}, we present the performance of one-shot versions of different defenses. In the main paper, we show the results for CIFAR10. Here, we present the performance for the other three datasets.

\begin{table}[t]
\begin{minipage}{0.375\textwidth}
    \centering
    \caption{\footnotesize \textbf{Impact of both weights and bias fine-tuning}. Up to now, we have only fine-tuned the weights. B We present the average drop in ASR and ACC over 14 attacks on CIFAR10.}
    \scalebox{0.8}{
    \begin{tabular}{c|c|c}
    \toprule
         Bias & Avg. ASR Drop & Avg. ACC Drop \\
         \midrule
         Frozen & 95.56 & \textbf{1.81} \\
         Unfrozen & \textbf{95.63} & 2.32\\ 
         \bottomrule
    \end{tabular}}

    \label{tab:bias_ft}
   
\end{minipage}
\hfill
\begin{minipage}{0.59\textwidth}
\vspace{-3.5mm}
\centering
    \caption{\footnotesize Performance of NFT with \textbf{different network architectures}. We consider both CNN and vision transformer (ViT). The CIFAR10 dataset has been used here.}
    
    \scalebox{0.8}{
    \begin{tabular}{l|cc|cc|cc|cc}
    \toprule
     Attack   & \multicolumn{4}{c|}{WaNet} & \multicolumn{4}{c}{LIRA}  \\ 
     \midrule
     Defense & \multicolumn{2}{c|}{No Defense} & \multicolumn{2}{c|}{With NFT} & \multicolumn{2}{c|}{No Defense} & \multicolumn{2}{c}{With NFT} \\
     \midrule
     Architecture &  ASR  & ACC &  ASR  & ACC &  ASR  & ACC &  ASR  & ACC \\
    \midrule
     VGG-16  &  97.45 & 91.73 &  2.75 & 89.58 &   99.14 & 92.28 &  2.46 & 90.61 \\
     EfficientNet   &  98.80 & 93.34 &  2.93 & 91.42 &   99.30 & 93.72 & 2.14 & 91.52    \\
     ViT-S & 99.40 & 95.10 &  3.63 & 93.58 & 100 & 94.90 &  1.98 & 93.26   \\
     \bottomrule
    \end{tabular}}
    \label{tab:diff_architectures}
    \vspace{-5mm}
\end{minipage}
\end{table}

\noindent{\bf Ablation Study on Hyper-parameters.}\label{sec:hyper-parameters}
To observe the impact of different hyper-parameters, we change the learning rate and batch size of NFT in Table~\ref{tab:learning_rate} and Table~\ref{tab:batch_size}. Upon observing the performance, we chose a batch size of 128 and 0.05 which gives us SOTA performance.

\noindent{\bf Combination of Backdoor Attack.}\label{sec:composite_backdoor}
To show the impact of NFT on more attack variations, we formulate a composite backdoor attack by combining 2/3 different attacks simultaneously. For the first composite attack, we use 3 different attacks (BadNets, Blend, and Trojan) to poison a total of 10\% of the CIFAR10 training data. As shown in Table~\ref{tab:composite_backdoor}, we have a combined attack success rate of 100\% and clean accuracy of 88.26\%. Both of the compared methods, MCR and ANP, perform worse than NFT in terms of ASR and ACC. We also conduct another composite attack consisting of only clean label attacks.

\noindent{\bf Effect of Bias Fine-tuning.} A study with frozen and unfrozen bias has been presented in Table~\ref{tab:bias_ft}. Freezing the bias results in better ACC with a slight trade-off in ASR.

\noindent \textbf{Different Network Architectures.} To validate the effectiveness of our method under different network settings. In Table~\ref{tab:diff_architectures}, we show the performance of NFT with some of the widely used architectures such as VGG-16~\cite{simonyan2014very}, EfficientNet~\cite{tan2019efficientnet} and Vision Transformer (VIT)~\cite{dosovitskiy2020image}. Here, we consider a smaller version of ViT-S with 21M parameters. NFT can remove backdoors irrespective of the network architecture. This makes sense as most of the architecture uses either fully connected or convolution layers, and NFT can be implemented in both cases.

\begin{table}[t]
    \centering
    \caption{\footnotesize Evaluation of \textbf{\emph{augmented defenses}} where we consider strong augmentations for all other defenses. \emph{A naive combination of strong augmentations and other defenses} is still not enough to outperform NFT.}
    
    \scalebox{0.8}{
    \begin{tabular}{c|cc|cc|cc|cc}
    \toprule
    Attacks& \multicolumn{2}{c|}{WaNet} & \multicolumn{2}{c|}{LIRA} & \multicolumn{2}{c|}{ISSBA} & \multicolumn{2}{c}{Dynamic} \\\midrule
    Methods& ASR & ACC& ASR & ACC& ASR & ACC& ASR & ACC\\\midrule
    No Defense & 98.64 & 92.29 & 99.25 & 92.15 & 99.80 & 92.78 & 100 & 92.52\\
    RNP-S & 4.12 & 84.10 & 5.75 & 86.26 & 5.53 & 83.90 & 3.24 & 86.50 \\
    FT-SAM-S & 2.96 & 88.34 & 3.93 & 89.08 & \textbf{3.91} & 88.12 & 1.76 & 85.86 \\
    NFT &\textbf{2.38}&\textbf{89.65}&\textbf{1.53}&\textbf{90.57}&{4.24}&\textbf{90.18}&\textbf{1.17}&\textbf{90.97}\\
    \bottomrule
    \end{tabular}}
    \label{tab:aug-defense}
\end{table}

\begin{table}[t]
    \centering
    \caption{\footnotesize Purification performance (\%) for \textbf{various validation data sizes}. NFT performs well even with a very small amount of clean data. Validation size 0.01\% indicates One-Shot NFT. In our main evaluation (Table 1 of main paper), we consider 1\% validation size. For evaluation, we use CIFAR10 and Dynamic attack.}
    
    \scalebox{0.8}{
    \begin{tabular}{c|cc|cc|cc|cc}
    \toprule
         Valid. Size & \multicolumn{2}{c|}{0.02\%}  & \multicolumn{2}{c|}{0.1\%} & \multicolumn{2}{c|}{0.2\%}  & \multicolumn{2}{c}{0.5\%}  \\
         \midrule
         Method & ASR & ACC& ASR & ACC& ASR & ACC& ASR & ACC \\
         \midrule
         No Defense & 100 & 92.52 &  100 & 92.52 &100 & 92.52 &100 & 92.52   \\
         ANP & 50.78 & 58.71 & 38.94 & 66.97 & 31.80 & 79.61 & 24.93 & 82.62   \\
         RNP & 13.66 & 70.18 & 8.35 & 82.49& 5.72 & 84.70& 3.78 & 85.26    \\
         \midrule
         NFT (Ours) & \textbf{6.91} & \textbf{83.10} & \textbf{3.74} & \textbf{89.90} & \textbf{1.61} & \textbf{90.08} & \textbf{1.45} & \textbf{90.84}  \\
         \bottomrule
    \end{tabular}}
    \label{tab:dif_val}
\end{table}

\begin{table}[t]
\centering
    \caption{ \textbf{Ablation Study} on \textbf{$\eta_c$}.}
    \scalebox{0.8}{
    \begin{tabular}{l|c|c|c|c|c|c}
    \toprule
     $\eta_c$ &  \textbf{1e-2} & \textbf{5e-3} & \textbf{1e-3} &  \textbf{5e-4} & \textbf{1e-4} & \textbf{5e-5} \\
    \midrule
     \textbf{Avg. ASR Drop}  & 94.3  & 94.6 & 95.1 &  95.6 & 95.6 & \textbf{95.7}   \\
     \textbf{Avg. ACC Drop}  & \textbf{1.46}  & 1.68 & 1.72 & 1.81 & 1.91 & 2.12 \\
     \bottomrule
    \end{tabular}}
    \label{tab:regularizer}
\end{table}

\noindent\textbf{Augmented Defenses.} In Table~\ref{tab:aug-defense}, we show the performance of augmented defenses where we consider Data Augmentations (like MixUp) for other defenses, \eg, RNP-S. Due to the adversarial perturbation-based algorithmic design, using augmentations for ANP and AWM, like RNP and FT-SAM, does not make sense. It can be seen that our proposed method can harness the power of augmentations better. Unlike other defenses, NFT is motivated by regular fine-tuning and aims to find the correct validation. We take a milder approach by indirectly changing the parameters using neural masks and ensuring that the parameter adjustment is not drastic.

\begin{figure}    
  \centering
    {\includegraphics[width=0.9\textwidth]{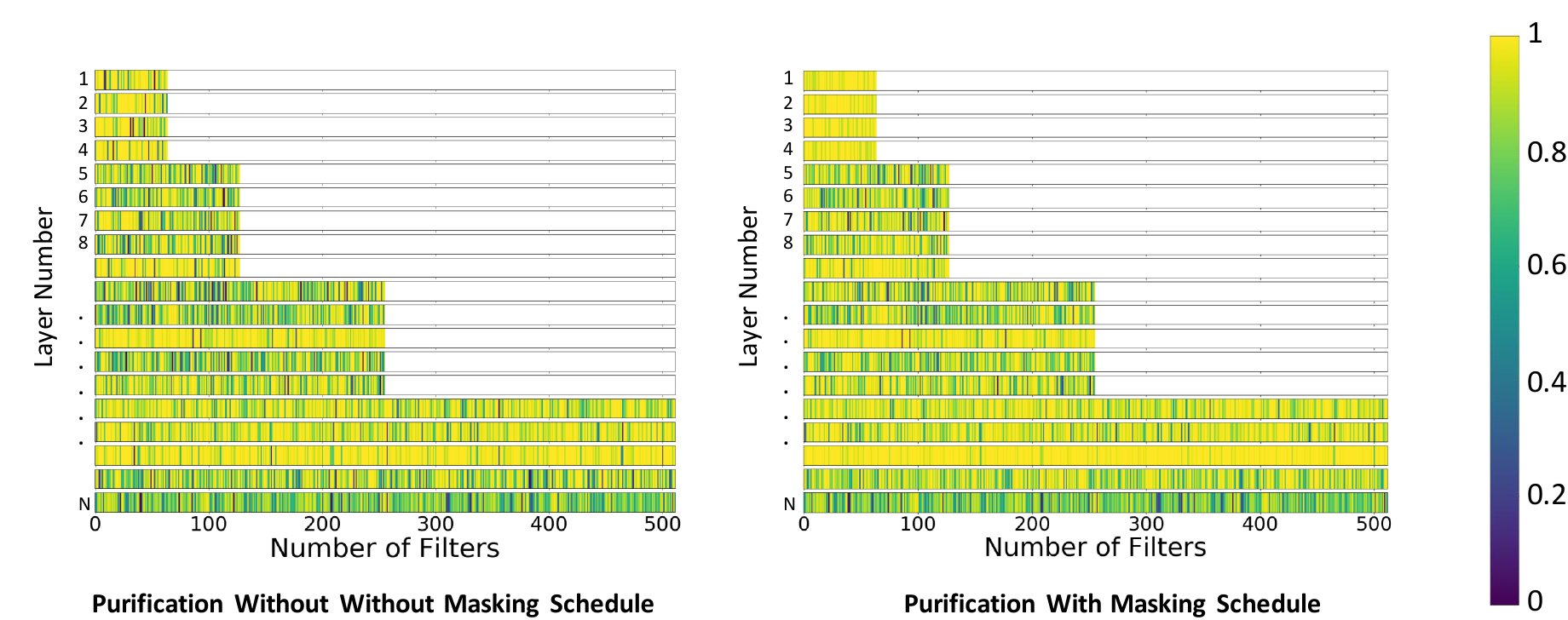}}\hfill  
    \caption{\footnotesize \textbf{Illustration of Mask Heatmap with and without scheduling function ($\mu$)}. This ablation is done for the LIRA attack and CIFAR10 dataset. In both cases, we do not use the mask regularizer here just to show the impact of the $\mu$. The first couple of layers have minimal changes.}
  \label{fig:heatmap_schedule}  
\end{figure}

\begin{figure}   
  \centering
    {\includegraphics[width=0.9\textwidth]{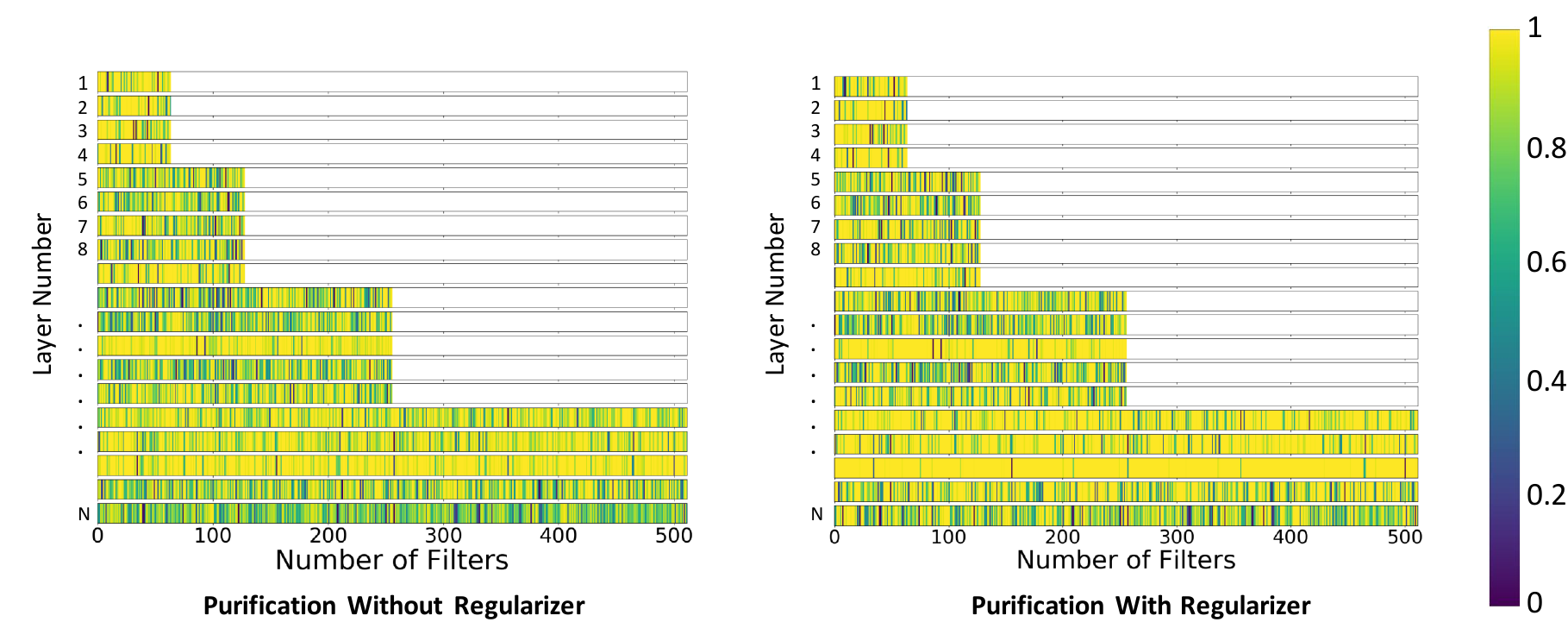}}\hfill  
    \caption{\footnotesize \textbf{Illustration of Mask Heatmap with and without regularizer}. This ablation is done for the Badnets attack and CIFAR10 dataset. In both cases, we do not use the mask scheduling function here just to show the impact of the regularizer. With the mask regularizer, we restrict the weights to be closer to the original backdoor model (shown by the overall larger yellow region).}
  \label{fig:heatmap_reg}  
\end{figure}

\noindent\textbf{Effect of Various Validation Size.} 
We also present how the total number of clean validation data can impact the purification performance. In Table~\ref{tab:dif_val}, we see t e change in performance while varying the validation size from $0.02\%\sim0.5\%$. Validation size 0.02\% indicates One-Shot NFT. In genera , we take 1\% of training samples as clean validation data. We consider the Dyn-one at ack on the CIFAR10 dataset for this evaluation. Even with only ten validation images, NFT can successfu ly remove the backdoor by reducing the attack success rate to $6.91\%$.

\noindent \textbf{Impact of $\eta_c$.} We study the impact of $\eta_c$ in Table~\ref{tab:regularizer}. Mask regularizer is useful in retaining lean accuracy (ACC) under severe validation data shortages. However, if we use a l rge value for $\eta_c$, the regularizer may prevent any change in the decision boundary altogether. As a result, the e fect of MixUp may be reduced significantly res lting in poor purification performance. Therefore, we use a suitable  alue for $\eta_c$ to ensure the optimal change in decision boundary, leadi g to a purified model with good ACC. 

\noindent\textbf{Mask Heatmap.} In Figure~\ref{fig:heatmap_schedule}-\ref{fig:heatmap_reg}, we show the mask hetmaps under different scenarios. Figure~\ref{fig:heatmap_schedule} shows the mask heatmaps with and without scheduling function ($\mu$). It can be seen that even with minimal changes to the first couple of layer weights, we could achieve purification.  This suggests that the backdoor affects the later hidden layers more, and our design of a mask scheduling function is well justified. Figure~\ref{fig:heatmap_reg}  shows the mask heatmaps with and without the mask regularizer. The regularizer keeps the purified model weights closer to the original backdoor model weights.

\end{document}